\documentclass[11pt]{article}

\usepackage[margin=1in]{geometry}
\usepackage{amsmath,amssymb,amsthm,mathtools}
\usepackage{bm}
\usepackage{upgreek}
\usepackage{algorithm}
\usepackage{algpseudocode}
\usepackage{thmtools}
\usepackage{hyperref}
\usepackage{cleveref}
\usepackage{mathtools}
\usepackage{xcolor}
\mathtoolsset{showonlyrefs=true}

\usepackage{natbib}

\hypersetup{colorlinks=true, linkcolor=blue, citecolor=blue, urlcolor=blue}

\newcommand{\rmd}{\mathrm{d}}
\newcommand{\R}{\mathbb{R}}
\newcommand{\E}{\mathbb{E}}
\newcommand{\dd}{\textup{d}}

\newcommand{\n}{\nabla}
\DeclareMathOperator{\sg}{\mathrm{sg}}
\newcommand{\cF}{\mathcal{F}}
\newcommand{\cP}{\mathcal{P}}
\newcommand{\iidsim}{\stackrel{\text{i.i.d.}}{\sim}}

\title{On the Wasserstein Gradient Flow Interpretation of Drifting Models}

\author{Arthur Gretton, Li Kevin Wenliang, Alexandre Galashov, James Thornton,\\ Valentin De Bortoli, Arnaud Doucet}

\date{Google DeepMind}

\begin{document}
\maketitle

\begin{abstract}
Recently, \citet{deng2026drifting} proposed Generative Modeling via Drifting (GMD), a novel framework for generative tasks. This note presents an analysis of GMD through the lens of Wasserstein Gradient Flows (WGF), i.e., the path of steepest descent for a functional in the space of probability measures, equipped with the geometry of optimal transport. 
Unlike previous WGF-based contributions, GMD can be thought of as directly targeting a fixed point of a specific WGF flow. We demonstrate three main results: first, that one algorithm proposed by \citeauthor{deng2026drifting} can be viewed as targeting the fixed point of the KL Wasserstein gradient
flow after replacing the intractable scores by Parzen-smoothed score estimates. Second, that the algorithm actually implemented by \citeauthor{deng2026drifting} corresponds to a different procedure, which bears some resemblance to the fixed point of a WGF on the Sinkhorn divergence, but lacks certain desirable properties of the latter.
Third, the same idea can be extended to the limiting point of other WGFs, including the Maximum Mean Discrepancy (MMD), the sliced Wasserstein distance, and GAN critic functions.

\end{abstract}

\section{Introduction}
Generative Modeling via Drifting (GMD), introduced by  \citet{deng2026drifting}, represents a novel and empirically successful approach to generative tasks. The method achieves strong performance without adversarial training, distillation, or multi-step sampling, relying instead on a non-parametric, mean-shift style drift operator to iteratively move generated samples towards the target data distribution during training. Despite these practical strengths, the underlying theoretical mechanics of GMD remain relatively unexplored. In this note, we discuss the relation between drifting models and Wasserstein Gradient Flows (WGF).  

We begin in \Cref{sec:GMD} with a short review of the Generative Modeling via Drifting (GMD) framework proposed by \citeauthor{deng2026drifting}, and then cover the basics of Wasserstein Gradient Flows (WGF) \citep{ambrosio2008gradient} in \Cref{subsec:WGF}. Next, in \Cref{sec:fixedpoint}, we illustrate a re-interpretation of one of the simpler GMD instantiations as a technique to find an approximation to a fixed point of a WGF for the Kullback--Leibler divergence. 

The algorithm  implemented in practice by \citeauthor{deng2026drifting} (Algorithm 2 in their paper) is more complex, however. In Section~\ref{sec:sinkhorn}, we analyse this algorithm (with minor variations), demonstrating a number of results. First,  we show that this model bears a resemblance to a drifting model on the Wasserstein gradient flow on the Sinkhorn divergence, and shares a specific property of this flow; thus, we refer to it as the Sinkhorn Proxy.  Second, we prove that the velocity field of the Sinkhorn Proxy is zero if and only if $p$ and $q$ agree when the Gaussian kernel is used, under additional finite-measure assumptions.  In addition to these positive results, we also present two negative results for the Gaussian kernel case: first, that the velocity field  does not in fact correspond to the Wasserstein gradient of {\em any} scalar distributional loss (unless a highly implausible condition holds); second, that convergence of the Sinkhorn Proxy does not inherit the good convergence behaviour one would expect of gradient flow on a true optimal transport distance, in the event that source and target distributions have distant modes with non-overlapping support.   

We next explore a broader perspective on drifting in Section \ref{sec:extensions}, beginning with a general-purpose formulation of a drifting algorithm. We then consider instantiations for specific measures of divergence.  \citet[][Appendix C.2]{deng2026drifting} discuss  drifting on the maximum mean discrepancy (MMD), which is very close in practice to their drifting algorithm, and represents the fixed point of the Wasserstein gradient flow on the MMD \citep{arbel2019mmdflow}. We further introduce drifting on the Sliced-Wasserstein flow \citep{cozzi2024swflow},  and on the class of GAN critic functions derived from f-divergences \citep{nowozin2016fgan}, providing among other results an alternative  to the  KL-drifting method of Section \ref{sec:fixedpoint}. 

{\bf Related work.} The interpretation of the drifting algorithm as the stationary point of a Wasserstein gradient flow was established by \citet{cao2026gradientflowdrifting,turan2026generativedriftingsecretlyscore,he2026sinkhorndrifting}.  \citet{cao2026gradientflowdrifting} describes a connection of drifting to the fixed point of a Wasserstein gradient flow on the KL between Parzen estimates of the densities, although there is a subtle difference with our presentation - we provide more detail in Section \ref{sec:fixedpoint}.   They further note the interpretation of MMD drifting as the stationary point of the MMD gradient flow.  \citet{turan2026generativedriftingsecretlyscore} and \citet{lai2026unifiedviewdriftingscorebased}  show that for the Gaussian kernel, the velocity field of the initial drifting instantiation is expressed as a score difference, which is zero only when the distributions agree.  \citet{he2026sinkhorndrifting} and \citet[][Section 6.3, Appendix I.4]{turan2026generativedriftingsecretlyscore} each describe a drifting algorithm built on the Wasserstein gradient flow of the fully converged Sinkhorn divergence.
Alternative perspectives on drifting models, complementary to the one presented here, are proposed by \citet{lai2026unifiedviewdriftingscorebased} and \citet{li2026long}.  
 It is notable that the WGF interpretation of the drifting
model actually implemented by
 \citet[][Algorithm 2]{deng2026drifting} is not covered in these prior works.

\section{A Wasserstein Gradient Flow Perspective} \label{sec:perspective}
\subsection{Generative Modeling via Drifting models}\label{sec:GMD}
Let $p$ be the target data distribution on $\R^d$. To approximate $p$, we use a neural network $f_\theta:\R^c \to \R^d$ whose input is some noise $\epsilon \sim \mu$ (e.g., $\epsilon \sim \mathcal{N}(0,I_c)$) and whose output is $x=f_\theta(\epsilon)$. The implicit distribution of $x$ is denoted $q_\theta$. We assume that $f_\theta(\epsilon)$ is differentiable w.r.t. $\theta$ for all $\epsilon \in \R^c$. 

The drifting-model paradigm evolves $\theta$ during training for $q_\theta$ to match the target $p$. This method relies on a drift field $V_{p,q_\theta}: \R^d \to  \R^d$ which we use to displace samples from our current model. We will discuss several instances of drifting fields in subsequent sections.

We  update parameters $\theta$ by minimizing the regression loss
\begin{equation}\label{eq:dt-loss}
\mathcal{L}(\theta)=\E_{\epsilon\sim \mu} \Big[\|f_{\theta}(\epsilon)-\sg\!\big(f_{\theta}(\epsilon) + \widehat{V}_{p,q_{\theta}}(f_{\theta}(\epsilon))\big)\|^2\Big];
\end{equation}
i.e. the average mean square error between a sample and its displaced version; see \cite[Equation (6)]{deng2026drifting}. Here $\widehat{V}_{p,q_{\theta}}$ is an approximation to $V_{p,q_{\theta}}$.

The stop-gradient $\sg(\cdot)$ operator appearing in \eqref{eq:dt-loss} is essential: it ensures that optimization moves the generator \emph{toward} the prescribed transport step,
rather than changing the drift estimator itself \citep{deng2026drifting}. Crucially, it also does not require back-propagating through $\widehat{V}_{p,q_\theta}$. Additionally, it is expected that this approach is more robust than minimizing $\mathbb{E}_{\epsilon}[\|\widehat{V}_{p,q_\theta}(f_\theta(\epsilon))\|^2]$ where the generator can reduce the loss by collapsing mass so attraction and repulsion cancel, spreading mass to make kernel weights flat, or shaping $q_\theta$ so that $\widehat{V}_{p,q_\theta}$ becomes small without matching $p$. In Section~\ref{sec:algo}, Proposition~\ref{prop:loss_gradient}, we show that the gradient of this loss turns out to be correct under the WGF perspective presented next, and the stop gradient operator is crucial.  

\subsection{Wasserstein Gradient Flows}\label{subsec:WGF}
We will  sketch here the basics of WGFs. We refer the reader to \citep{ambrosio2008gradient,santambrogio2016overview} for a comprehensive treatment and to Appendix A in \citep{arbel2019mmdflow} for a gentle introduction\footnote{Throughout the  derivation of WGF for KL below, we assume that the distributions under discussion admit positive $C^1$ densities with sufficient decay for the integrations by parts and score fields to be valid. For parametric generators this can be ensured, for example, by taking $c=d$ and assuming $f_\theta$ is a sufficiently regular diffeomorphism, or by adding noise. If $c<d$ or $f_\theta$ is not regular enough, $q_\theta$ may be singular with respect to Lebesgue measure; then the density-level KL derivations should be interpreted formally or after smoothing. The MMD and Sinkhorn statements below are stated at the level of measures whenever possible.}

Let $\cP_2(\R^d)$ be the space of distributions with finite second-order moments on $\R^d$.
Let $\cF_p:\cP_2(\R^d)\to\R$ be an energy functional minimized at $q=p$; e.g. the Kullback--Leibler divergence $\cF_p(q)=\textup{KL}(q\|p)$ or the squared Maximum Mean Discrepancy $\cF_p(q)=\mathrm{MMD}_{k}^{2}(p,q)$ for kernel $k$ (see \cref{sec:otherDivergenceMeasures}).
Under standard regularity assumptions, the Wasserstein-2 ($W_2$) gradient flow of $\cF_p$ is characterized by the partial differential equation (PDE) \citep{ambrosio2008gradient,santambrogio2016overview}
\begin{equation}\label{eq:w2-gf}
\partial_t q_t (x)=\; \n\cdot\left(q_t(x) \n \frac{\delta \cF_p}{\delta q_t}(x)\right),
\end{equation}
where  $\frac{\delta \cF_p}{\delta q}:\R^d\to\R $ is the first variation of the functional $\cF_p$ with respect to the density $q$ and $\nabla \cdot$ is the divergence operator. If it exists, this first variation is defined as the unique function (up to an additive constant) that satisfies the following linear expansion for any perturbation $\chi$ in the tangent space (i.e. $\int \chi(x)\dd x=0, \int \|x\|^2 |\chi(x)|\dd x< \infty$), where $h \in \R$
\begin{equation}
\left. \frac{\dd}{\dd h} \right|_{h=0} \cF_p(q+ h \chi) = \int \frac{\delta \cF_p}{\delta q}(x)\,\chi(x)  \dd x.
\end{equation}
The PDE \eqref{eq:w2-gf} implies that $\cF_p(q_t)$ is non-increasing over time as, under regularity conditions, we have
\begin{align}
\frac{\dd}{\dd t} \cF_p(q_t) &= \int \frac{\delta \cF_p}{\delta q_t}(x)\,\partial_t q_t(x) \, \dd x\\
&=\int \frac{\delta \cF_p}{\delta q_t}(x)\, \nabla \cdot \left( q_t(x) \nabla \frac{\delta \cF_p}{\delta q_t}(x) \right) \dd x\\
&=- \int q_t(x) \left\| \nabla \frac{\delta \cF_p}{\delta q_t}(x) \right\|^2 \dd x \leq 0,
\end{align}
where we have used first the chain rule, then \eqref{eq:w2-gf} and finally the divergence theorem.

The PDE \eqref{eq:w2-gf} can be  rewritten as a continuity equation
\begin{equation}\label{eq:continuity}
\partial_t q_t(x) + \n\cdot(q_t(x) V_{p,q_t}(x))=0,
\qquad
V_{p,q_t}(x)= - \n \frac{\delta \cF_p}{\delta q_t}(x).
\end{equation}
See Figure~\ref{fig:wgf_example} for an illustration of relevant quantities using the KL divergence as $\mathcal{F}_p$. However, note that this is not a standard continuity equation as the drift not only depends on $x$ but also on the distribution $q_t$ of $x_t$ at time $t$.  
Formally, this continuity equation corresponds to the
following McKean--Vlasov process for individual samples,
\begin{equation}\label{eq:particle-ode}
\dd x_t = V_{p,q_t}(x_t) \dd t,\qquad x_0 \sim q_0;
\end{equation}
i.e., this ODE is such that $x_t \sim q_t$. In order to implement  such an ODE, we need to discretize it and approximate its drift.
For example, using an Euler method leads to
\begin{equation}\label{eq:EulerODE}
x_{t+\eta}=x_t+\eta \widehat{V}_{p,q_t}(x_t),
\end{equation}
where $\eta > 0$ is a stepsize, and we use $\widehat{V}_{p,q_t}$ estimated from finite samples, which depends on the energy functional being minimized. 

\subsection{Drifting Models as Fixed Point of a Wasserstein Gradient Flow on the KL}\label{sec:fixedpoint}
We provide two alternative views of Algorithm~1 in \citet{deng2026drifting} through the score function of distributions. We also clarify the distinction between the pointwise Parzen-smoothed score-difference field and the true Wasserstein gradient flow of $\mathrm{KL}(q_\tau\|p_\tau)$. The Sinkhorn Proxy method, which better resembles \citeauthor{deng2026drifting}'s Algorithm~2 (main method), is introduced in Section~\ref{sec:sinkhorn}. We then  show common failure modes of the Sinkhorn Proxy and score objective.

\begin{figure}
\centering
   \includegraphics[width=0.7\textwidth]{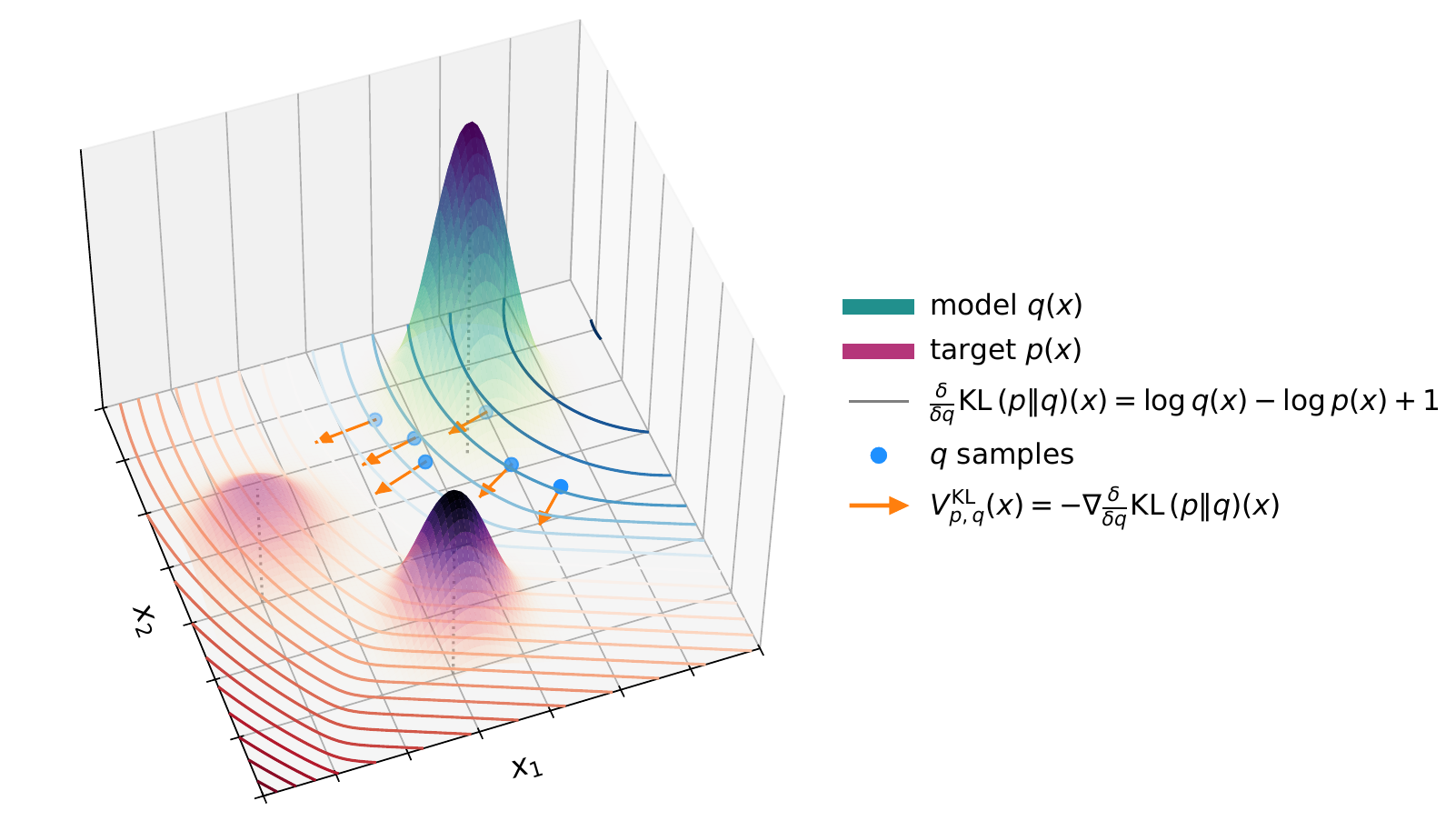}
 \caption{An example of the WGF associated with $\mathcal{F}_p(q)=\mathrm{KL}(q\|p)$, where $p$ is the target distribution and $q$ the model distribution. The contour shows the first variation $\delta\mathrm{KL}(q\|p)/\delta q=\log q-\log p+1$, where blue is positive and pink is negative. The arrows show the WGF velocity vectors evaluated at samples from $q$, namely $V^{\mathrm{KL}}_{p,q}=\nabla\log p-\nabla\log q$. In practice, the flow vectors are estimated from samples from unknown $p$ and $q$.
 \label{fig:wgf_example}
 }
\end{figure}
\subsubsection{A score-difference drift field}\label{sec:kl_drift}
For probability densities $p,q$ on $\R^d$, we first consider the following drift field defined by \citet[eq. (8)]{deng2026drifting}:
\begin{equation}\label{eq:driftingfield}
V_{p,q}(x)=V_p(x)-V_q(x)
\end{equation}
with
\begin{equation}\label{eq:meanshift}
V_p(x)=\frac{\E_{y\sim p}\big[k_{\tau}(x,y)\,y\big]}{\E_{y\sim p}\big[k_{\tau}(x,y)\big]}-x,\qquad V_q(x)=\frac{\E_{x'\sim q}\big[k_{\tau}(x,x')\,x'\big]}{\E_{x'\sim q}\big[k_{\tau}(x,x')\big]}-x
\end{equation}
where $k:\R^d\times\R^d\to(0,\infty)$ is a kernel.\footnote{We emphasize that this is not the algorithm \citeauthor{deng2026drifting} use in practice, which has a different interpretation, and will be discussed in Section \ref{sec:sinkhorn}.} This defines a \emph{mean-shift} style drifting field \eqref{eq:driftingfield} built from an ``attraction'' term
(from positives $y\sim p$) and a ``repulsion'' term (from negatives $x'\sim q$). 
In particular, for the Gaussian kernel in \cite[Section 3.3]{deng2026drifting} defined as
\begin{equation}\label{eq:rbf}
k_\tau(x,y)\;=(\pi \tau)^{-d/2}\;\exp\!\Big(-\tau^{-1}\|x-y\|^2\Big),
\qquad \tau>0,
\end{equation}
we have for $p_\tau(x)\triangleq \E_{y \sim p}[k_\tau(x,y)]$ and $q_\tau(x)\triangleq \E_{x' \sim q}[k_\tau(x,x')]$ that
\begin{align}\label{eq:Tweedie}
\n \log p_\tau(x)
&=\frac{\n p_\tau(x)}{p_\tau(x)}
=\frac{2}{\tau}\frac{\E_{y \sim p}\big[(y-x)k_\tau(x,y)\big]}{\E_{y \sim p}\big[k_\tau(x,y)\big]}
=\frac{2}{\tau}V_p(x)
\end{align}
and similarly $\n \log q_\tau(x)=\frac{2}{\tau}V_q(x)$.  Note that \eqref{eq:Tweedie} is an instance of Tweedie's formula.
Therefore we have for \eqref{eq:rbf} that 
\begin{equation}\label{eq:v-score-diff}
V_{p,q}(x)
\;=\;
\frac{\tau}{2}\Big(\n\log p_\tau(x)-\n\log q_\tau(x)\Big);
\end{equation}
that is, up to a constant factor, the drifting field is the difference of the scores of noised versions of $p$ and $q$. This velocity field and its relation to the KL flow was proposed by
\citet[Section 3]{weber2023score}, and  later by
\citet[][Theorem 1]{lai2026unifiedviewdriftingscorebased} and \citet[][Theorem 4.1]{turan2026generativedriftingsecretlyscore}, where the latter also prove that for the Gaussian kernel, $V_{p,q}=0$ iff $p=q$ (their Theorem 5.1). \citet{franz2026drifting} show that when a non-Gaussian kernel is used, the  velocity field does not correspond to a Wasserstein gradient, and propose an alternative normalization which restores conservatism. 

Practically, we are interested in $V_{p,q_\theta}(x)=V_p(x)-V_{q_\theta}(x)$. Here $V_p$ and $V_{q_\theta}$ are approximated by sampling mini-batches $\{y_j\}_{j=1}^{N^+} \iidsim p$ and $\{x_i\}_{i=1}^{N^-} \iidsim q_\theta$. We define normalized weights:
\begin{equation}\label{eq:orig-weights}
\alpha_j^+(x)=\frac{k_{\tau}(x,y_j)}{\sum_{\ell=1}^{N^+}k_{\tau}(x,y_{\ell})},
\qquad
\alpha_i^-(x)=\frac{k_{\tau}(x,x_i)}{\sum_{\ell=1}^{N^-}k_{\tau}(x,x_{\ell})}.
\end{equation}
Then, from \eqref{eq:meanshift}, we have the drift estimate
\begin{equation}\label{eq:orig-means}
\widehat{V}_{p,q_\theta}(x)=\widehat{V}_p(x)-\widehat{V}_{q_\theta}(x)
~~\text{with}~~\widehat{V}_p(x)=\sum_{j=1}^{N^+}\alpha_j^{+}(x)(y_j-x),
\quad
\widehat{V}_{q_\theta}(x)=\sum_{i=1}^{N^-}\alpha_i^{-}(x)(x_i-x).
\end{equation}

\subsubsection{Wasserstein Gradient Flow for Kullback--Leibler (KL) divergence}\label{sec:kl_flow}
Consider the WGF associated to the Kullback--Leibler discrepancy $\cF_p(q)=\textup{KL}(q||p)$.
In this case, it can be easily checked that 
\begin{equation}\label{eq:firstvariationKL}
\frac{\delta \cF_p}{\delta q}(x)= \log q(x) - \log p(x) + 1
\end{equation}
so the ODE \eqref{eq:particle-ode} becomes
\begin{equation}\label{eq:LangevinODE}
\dd x_t=(\nabla \log p(x_t)-\nabla \log q_t(x_t))\dd t,\quad x_0 \sim q_0
\end{equation}
i.e. the drift of \eqref{eq:LangevinODE} is given by the difference of the (intractable) scores of both the data distribution $p$ and of $q_t$.\footnote{This ODE is the probability flow version of the following Langevin diffusion
\begin{equation}\label{eq:LangevinSDE}
\dd x_t=\nabla \log p(x_t)\dd t+\sqrt{2}\dd B_t,\quad x_0 \sim q_0
\end{equation}
i.e., the processes defined in \eqref{eq:LangevinODE} and \eqref{eq:LangevinSDE} have the same marginal distributions.} %

An Euler discretization (i.e. \eqref{eq:EulerODE}) of \eqref{eq:LangevinODE} with step size $\eta=\tau/2$ is given by 
\begin{equation}\label{eq:EulerGDM}
x_{t+\frac{\tau}{2}}=x_t+\frac{\tau}{2} (\nabla \log p(x_t)-\nabla \log q_t(x_t)),\qquad x_0 \sim q_0.
\end{equation}
Now, in particular, this shows that for $x_t \sim q_t=p$, the particle displacement over the time interval of size $\frac{\tau}{2}$ is given by $\frac{\tau}{2}(\nabla \log p(x_t)-\nabla \log q_t(x_t))$, which evaluates exactly to zero. The key point of GMD is that it is not trying to simulate the flow $(x_t)_{t\geq 0}$. It just directly attempts to find its limiting point by setting the following fixed-point equation. For $x_\infty \sim q=p$, the displacement is  zero, and we have from \eqref{eq:EulerGDM} that
\begin{equation}\label{eq:fixedpoint}
x_\infty =x_\infty+\frac{\tau}{2}(\nabla \log p(x_\infty)-\nabla \log q(x_\infty)),
\end{equation}
where both $\nabla \log p$ and $\nabla \log q$ are intractable. By making the approximations $\nabla \log p(x) \approx \nabla \log p_\tau(x)$ and $\nabla \log q(x) \approx \nabla \log q_\tau(x)$, we thus get the approximation of fixed-point equation \eqref{eq:fixedpoint}
\begin{equation}\label{eq:approxfixedpoint}
x_\infty =x_\infty+V_{p,q}(x_\infty),
\end{equation}
where $V_{p,q}(x)$ is defined in \eqref{eq:v-score-diff}.

Going back to the parametric framework, this suggests identifying $\theta$ as solving the fixed-point equation \eqref{eq:approxfixedpoint} using
\begin{equation}
f_\theta(\epsilon) =f_\theta(\epsilon)+V_{p,q_\theta}(f_\theta(\epsilon)),\qquad \epsilon \sim \mu;
\end{equation}
see \cite[Equation (4)]{deng2026drifting}. This is achieved by performing a Monte Carlo approximation $\widehat{V}_{p,q_\theta}$ in \eqref{eq:orig-means} of $V_{p,q_\theta}$ and minimizing \eqref{eq:dt-loss}. 
So, to summarize, drifting models for the drift field defined by \eqref{eq:driftingfield} and \eqref{eq:meanshift} can be re-interpreted as a method for approximating the fixed point of the KL WGF after the intractable scores $\nabla\log p$ and $\nabla\log q$ are replaced by the Parzen-smoothed score estimates $\nabla\log p_\tau$ and $\nabla\log q_\tau$. This should not be confused with the WGF of $\mathrm{KL}(q_\tau\|p_\tau)$, however, whose velocity is given by \eqref{eq:trueGradientFlowParzen} below.

We now review the related finding of \citet[][Corollary 4.10]{cao2026gradientflowdrifting}, which states that the drifting model's velocity field is the Wasserstein-2 gradient flow of the KL divergence between smoothed densities, $\cF^{(\tau)}_p(q):=\textup{KL}(q_\tau || p_\tau)$.    \citet[][eq. 7]{cao2026gradientflowdrifting} claim that \eqref{eq:v-score-diff} represents this velocity field. 
The correct expression for the velocity field of this KL divergence is not \eqref{eq:v-score-diff}, however, but
\begin{equation}\label{eq:trueGradientFlowParzen}
    V^{(\tau)}_{p,q}(y) = -\nabla_y \int k_{\tau}(x,y) \log \frac{q_\tau(x)}{p_\tau(x)} \dd x.
\end{equation}
The proof is in Appendix~\ref{sec:parzen_smoothed_kl}.  More precisely, if the kernel $k_\tau(x,y)= \bar{k}_\tau(x-y)$ for some function $\bar{k}$,  then we have the following result using a change of variables:
\begin{equation}
\label{eq:trueGradientFlowParzenmod}
    V^{(\tau)}_{p,q}(y) = \int k_\tau(x,y) \nabla \log p_\tau(x) \rmd x - \int k_\tau(x,y) \nabla \log q_\tau(x) \rmd x \neq \frac{\tau}{2}\left( \nabla \log p_\tau(y) - \nabla \log q_\tau(y)\right). 
\end{equation}

The same issue applies for the additional flows on Parzen window $f$-divergences defined in that work. As a downstream consequence, the dissipation result of \citet[][Corollary 4.8 (ii)]{cao2026gradientflowdrifting} is not established, since particles following \eqref{eq:v-score-diff}  are not following the velocity field for $\cF^{(\tau)}_p(q)$.

\section{A Sinkhorn Proxy for Drifting Models}\label{sec:sinkhorn}

The score-difference drift of \Cref{sec:fixedpoint} approximates scores via kernel density estimates. The algorithm used in practice by \citet{deng2026drifting} follows a different procedure, however (Algorithm~2 in their paper). We develop a closely related drift field---which we term the \emph{Sinkhorn Proxy}---by connecting it to the Wasserstein gradient flow on the \emph{debiased Sinkhorn divergence} \citep{Ramdas17wassersteinTest,feydy2019sinkhorn}:
\begin{equation}\label{eq:sinkhorn_div}
    S_\tau(q, p) = \textup{OT}_\tau(q, p) - \tfrac{1}{2}\,\textup{OT}_\tau(q, q) - \tfrac{1}{2}\,\textup{OT}_\tau(p, p),
\end{equation}
where $\textup{OT}_\tau$ denotes entropy-regularized optimal transport as defined by \citet{cuturi2013sinkhorn}, with regularization $\tau > 0$. This section is a summary of a more in-depth exposition in Appendices~\ref{sec:sinkhornProxyDetails}-\ref{sec:failureSyntheticProofs}.

Consider a set of $N^-\in\mathbb{N}_+$  samples $\{x_i\}_{i=1}^{N^-}$  drawn from the model,
and a set of $N^+\in\mathbb{N}_+$ data samples $\{y_j\}_{j=1}^{N^+}$. The WGF associated with $S_\tau$ can be estimated as follows.

\begin{restatable}[Exact Sinkhorn particle velocity]{prop}{propSinkhornVelocity}\label{prop:sinkhorn_velocity}
Let 
$q = \frac{1}{N^-}\sum_{i=1}^{N^-} \delta_{x_i}$
and
$p = \frac{1}{N^+}\sum_{j=1}^{N^+} \delta_{y_j}$
be empirical measures comprising samples from model and the data set, respectively. Let $W^+ \in \R_+^{N^-\times N^+}$ and $W^- \in \R_+^{N^-\times N^-}$ be the row-stochastic conditional optimal transport plans for $\textup{OT}_\tau(q,p)$ and $\textup{OT}_\tau(q,q)$, respectively, under the transport cost $C(x,y) = \frac{1}{2} \|x-y\|^2$. Then the Wasserstein gradient flow velocity of the debiased Sinkhorn divergence $S_\tau(q,p)$ at location $x_i$ is
\begin{equation}\label{eq:sinkhorn_exact}
    \widehat{V}^\textup{S}_{p,q}(x_i) = \sum_{j=1}^{N^+} W^+_{ij}\, y_j - \sum_{k=1}^{N^-} W^-_{ik}\, x_k.
\end{equation}
\end{restatable}
\noindent For reference and to make the paper self-contained, a proof is given in Appendix~\ref{sec:sinkhornProxyDetailedIntro}, using  results from \citep{feydy2019sinkhorn}. The flow velocity depends on $x_i$  through $W^+$ and $W^-$ (and the self-term in the second sum). \cite{he2026sinkhorndrifting} and \citet[][Section 6.3, Appendix I.4]{turan2026generativedriftingsecretlyscore}  propose a drifting algorithm built on the fully 
converged Sinkhorn solution.

\subsection{The empirical Sinkhorn proxy}
Computing \eqref{eq:sinkhorn_exact} requires running the Sinkhorn algorithm to convergence at each training step.
To follow \citeauthor{deng2026drifting}'s Algorithm 2,
we instead replace this with a one-shot approximation. Let $z^+_{im} = -\|x_i - y_m\|^2/\tau$ be the scaled affinities between the $i$'th model sample and $m$'th data point, and $z^-_{im} = -\|x_i - x_m\|^2/\tau$ be the scaled affinities within two samples. We define the pseudo-plan
\begin{equation}\label{eq:one_shot}
 \begin{aligned}
    {A}^+_{im} &= \sqrt{\frac{\exp(z^+_{im})}{(N^+)^{-1}\sum_{u=1}^{N^+} \exp(z^+_{iu})} \cdot \frac{\exp(z^+_{im})}{(N^-)^{-1}\sum_{v=1}^{N^-} \exp(z^+_{vm})}}\;,
    \\
    {A}^-_{im} &= \sqrt{\frac{\exp(z^-_{im})}{(N^-)^{-1}\sum_{u=1}^{N^-} \exp(z^-_{iu})} \cdot \frac{\exp(z^-_{im})}{(N^-)^{-1}\sum_{v=1}^{N^-} \exp(z^-_{vm})}}\;,
    \end{aligned}
\end{equation}
i.e., the geometric mean of the row-wise and column-wise softmax normalizations, with additional $N^+$ and $N^-$ scaling for well-defined population limits.    This amounts to composing a {\em simultaneous Sinkhorn/Ruiz half-iteration} from each side (Appendix~\ref{sec:whyCloseToSinkhorn}) - when $N^+=N^-$ and omitting the $N^+ N^-$ scaling,  repeated iteration of this procedure on $A^+$  converges to the entropy-regularised transport map  \citep[][Theorem 3.1]{knight2014symmetry}.
In a related finding, \citet[Proposition 3.3]{he2026sinkhorndrifting} show that the simpler drifting method described in \cref{sec:fixedpoint}
 corresponds to a gradient flow on the first step of the standard, alternating Sinkhorn-Knopp algorithm.

\paragraph{Cross-weighting.}
Unlike the fully converged solution, both ${A}^-$ and ${A}^+$ no longer have unit row or column sums.
\citeauthor{deng2026drifting} address the issue via a cross-weighting step, which we describe below. Let $s^+_i = (N^+)^{-1} \sum_j {A}^+_{ij}$ and $s^-_i = (N^-)^{-1} \sum_k {A}^-_{ik}$ be the respective normalized row sums, and set
\begin{equation}\label{eq:cross_weight}
    \widetilde{W}^+_{ij} = \frac{1}{N^+} {A}^+_{ij}\cdot s^-_i,
    \qquad
    \widetilde{W}^-_{ik} =  \frac{1}{N^-}  {A}^-_{ik}\cdot s^+_i.
\end{equation}
Writing $\tilde{\pi}^+_{ij} = (N^+)^{-1}  {A}^+_{ij}/s^+_i$,  $\tilde{\pi}^-_{ik} = (N^-)^{-1} {A}^-_{ik}/s^-_i$, and $Z_i=s_i^+ s_i^-$,  the empirical Sinkhorn proxy drift at $x_i$ becomes
\begin{equation}\label{eq:proxy_drift}
    \widehat{V}^\textup{SP}_{p,q}(x_i) = Z_i \left(\sum_{j=1}^{N^+} \tilde{\pi}^+_{ij}\, y_j - \sum_{k=1}^{N^-} \tilde{\pi}^-_{ik}\, x_k\right).
\end{equation}
which shares with the true Sinkhorn flow \eqref{eq:sinkhorn_exact} the property that the rows of $\tilde{\pi}^+_{ij}$ and $\tilde{\pi}^-_{ik}$ sum to 1, and that the $x_i$ cancels across both terms.

This drift for a parametric $q_\theta$ is then used to optimize the $\theta$, similar to the method in \cref{sec:kl_flow} for the KL drift in \cref{sec:kl_drift}.
The complete procedure for computing the drift is summarized in Algorithm~\ref{alg:sinkhorn_proxy}, where we  omit the $N^-$ and $N^+$ terms, which simply result in a common batch-size-dependent scaling factor on the step size $\eta$.\footnote{ We  do not address here the \texttt{ignore self} step of \citet[][Algorithm 2]{deng2026drifting}, where self-interactions are removed, which is important in the implementation.} As indicated in the algorithm, the Sinkhorn Proxy differs from the practical algorithm of \citet{deng2026drifting} in the row-normalization scope: we normalize ${A}^+$ and ${A}^-$ over $N^+$ and $N^-$ columns respectively, whereas \citeauthor{deng2026drifting} first concatenate positive and negative logits and normalize jointly over $N^+ + N^-$ columns.  Our formulation more closely matches the separate marginal structures of $\mathrm{OT}_\tau(q,p)$ and $\mathrm{OT}_\tau(q,q)$ (Appendix~\ref{sec:oneShot}).

\begin{algorithm}[ht]
\caption{Computing the Sinkhorn Proxy drift field $\widehat{V}$}
\label{alg:sinkhorn_proxy}
\begin{algorithmic}[1]
\Require Model samples $\{x_i\}_{i=1}^{N^-}$, target samples $\{y_j\}_{j=1}^{N^+}$, 
temperature ~$\tau$.
\State $z^+_{ij} = -\|x_i - y_j\|^2/\tau$, \;\; $z^-_{ik} = -\|x_i - x_k\|^2/\tau$ \Comment{compute affinities}
\State Iterations are executed for $j=1,\ldots,N^{+}$, $k=1,\ldots,N^{-}$ and $m =1,\ldots,N^{+}+N^{-}$ ranges.
\end{algorithmic}
\begin{minipage}[t]{0.48\linewidth}
\centering\textsc{Sinkhorn Proxy (ours)}\\[4pt]
\raggedright
\begin{algorithmic}[1]
\setcounter{ALG@line}{1}
  \For{$(\diamond, \ell) \in \{(+, j), (-,k)\}$}
    \State $\alpha^{\textup{row},\diamond}_{i \ell} = \frac{\exp(z^\diamond_{i \ell})}{\sum_{u=1}^{N^\diamond} \exp(z^\diamond_{iu})}$
    \State $\alpha^{\textup{col},\diamond}_{i \ell} = \frac{\exp(z^\diamond_{i \ell})}{\sum_{v=1}^{N^-} \exp(z^\diamond_{v \ell})}$
    \State ${A}^\diamond_{i \ell} = \sqrt{\alpha^{\textup{row},\diamond}_{i \ell} \cdot \alpha^{\textup{col},\diamond}_{i \ell}}$
  \EndFor
\end{algorithmic}
\end{minipage}%
\hfill
\vrule
\hfill
\begin{minipage}[t]{0.48\linewidth}
\centering\textsc{\citeauthor{deng2026drifting}, Algorithm~2 (DA2)}\\[4pt]
\raggedright
\begin{algorithmic}[1]
\setcounter{ALG@line}{1}
  \State $z = [z^+, z^-] \in \R^{N^- \times (N^+ + N^-)}$ \Comment{row concat}
  \State $\alpha^{\textup{row}}_{im} = \frac{\exp(z_{im})}{\sum_{u=1}^{N^+ + N^-} \exp(z_{iu})}$
  \State $\alpha^{\textup{col}}_{im} = \frac{\exp(z_{im})}{\sum_{v=1}^{N^-} \exp(z_{vm})}$
  \State ${A}_{im}\!=\!\sqrt{\alpha^{\textup{row}}_{im} \!\cdot\! \alpha^{\textup{col}}_{im}}$;\;
  \State $[{A}^+\!, {A}^-] = \operatorname{split\_column}({A}, [N^+, N^-])$ \hfill 
\end{algorithmic}
\end{minipage}
\begin{algorithmic}[1]
\setcounter{ALG@line}{6}
\State $s^+_i = {\sum_{j=1}^{N^+}} {A}^+_{ij}$, \;\; $s^-_i = {\sum_{k=1}^{N^-}} {A}^-_{ik}$
\State $\widetilde{W}^+_{ij} = {A}^+_{ij} \cdot s^-_i$, \;\; $\widetilde{W}^-_{ik} = {A}^-_{ik} \cdot s^+_i$ \Comment{cross-weighting}
\State \Return $\widehat{V}(x_i) = {\sum_{j=1}^{N^+}} \widetilde{W}^+_{ij}\, y_j - {\textstyle\sum_{k=1}^{N^-}} \widetilde{W}^-_{ik}\, x_k$
\end{algorithmic}
\end{algorithm}

\subsection{Analysis of the Sinkhorn Proxy}

  We now introduce the population counterpart to the velocity field induced by the one-shot approximation. We show two main results: first, that the velocity is zero if and only if $p=q$ under certain assumptions on $p,q,k$. Second, that the velocity field of the Sinkhorn proxy does not in general correspond to a Wasserstein gradient flow.

 Let $k(x, y) = \exp(-\|x-y\|^2 / \tau)$ be the unnormalized Gibbs kernel.\footnote{Note that this is now a positive definite kernel, not a Parzen window, so we do not employ the normalizer in \eqref{eq:rbf}.  We henceforth write $ p_{\tau}$ and $ q_{\tau}$ to denote convolutions with the unnormalized kernel, with some abuse of notation.} 
In the continuous population limit, recall\footnote{In the original algorithm, the logits for the positive and negative targets are concatenated, which would give a term $  \int k(x,x') \dd (p+q)(x')  = (k * (p+q))(x)$ in  place of $p_{\tau}(x)$.  A closer approximation to the Sinkhorn is attained in the equations we use below.}
\[
    p_{\tau}(y)   = (k * p)(y),  \qquad 
    q_{\tau}(x) = (k * q)(x). 
\]
The geometric mean proxy affinities can be written symmetrically as:
\begin{equation}
    A^+(x,y) = \frac{k(x,y)}{\sqrt{q_{\tau}(y)p_{\tau}(x)}}, \quad A^-(x,x') = \frac{k(x,x')}{\sqrt{q_{\tau}(x) q_{\tau}(x')}}.
\end{equation}
The row integrals of the independent proxies are 
\begin{equation}
s^+(x) = \int A^+(x,y) \dd p(y),
\qquad
s^-(x) = \int A^-(x,x') \dd q(x').
\end{equation}
The cross-weighting step  defines the final update weights as 
\begin{equation}
W^+(x,y) = s^-(x) A^+(x,y),
\qquad 
W^-(x,x') = s^+(x) A^-(x,x').
\end{equation}
The population drift vector for the Sinkhorn Proxy is 
\begin{equation}\label{eq:initialSinkhornProxyDrift}
V^{\mathrm{SP}}_{p,q}(x) = \int W^+(x,y) y \, \dd p(y) - \int W^-(x,x') x' \, \dd q(x'). 
\end{equation}

The velocity field \eqref{eq:initialSinkhornProxyDrift} may be written in the following alternative form, as proved in Appendix \ref{sec:population_counterpart}.
\begin{restatable}[Population form of the Sinkhorn Proxy]{prop}{propScoreInterpretation}\label{thm:scoreInterpretationDrifting}
Let  
$\dd \tilde{p}(y) = q^{-1/2}_{\tau}(y) \dd p(y)$,
and $\dd \tilde{q}(x) = q^{-1/2}_{\tau}(x) \dd q(x)$ (if $p$ and $q$ have densities, this is equivalent to the density notation $\tilde p(y)=p(y)/\sqrt{q_\tau(y)}$ and $\tilde q(y)=q(y)/\sqrt{q_\tau(y)}$).
Then the Sinkhorn Proxy velocity field \eqref{eq:proxy_drift} with Gaussian kernel has the population-limit form
\begin{equation}\label{eq:proxy_score}
    V^{\mathrm{SP}}_{p,q}(x) = \frac{\tau}{2}\, Z(x) \Big( \n \log (k * \tilde{p})(x) - \n \log (k * \tilde{q})(x) \Big),
\end{equation}
where $Z(x) = s^+(x)\,s^-(x)$ is a spatially varying pre-conditioner.
\end{restatable}
This reveals a structural resemblance to the KL score-difference drift of \Cref{sec:fixedpoint}, modulated by $Z(x)$.
The velocity field is \emph{consistent}, as shown in Appendix \ref{sec:zeroIFFPisQ}.
\begin{restatable}[Consistency]{thm}{thmConsistency}\label{thm:zeroIFFPisQ}
Assume that the weighted measures $\tilde p$ and $\tilde q$ have finite mass, and that
a Gaussian kernel $k$ is used. Then
 $V = 0$   for all $x\in\mathbb{R}^d$ if and only if $p = q$.
\end{restatable}
The finite mass of $\tilde{p}(y) = p(y)/\sqrt{q_{\tau}(y)}$ is not guaranteed. For instance, when $q$ is light-tailed, $p$ is heavy-tailed,  both are supported on $\mathbb{R}^d$, and the kernel $k$ performs insufficient blurring on $q$,  then the assumption can fail. 

As we show in Appendix \ref{sec:curlProof}, however, $V^{\mathrm{SP}}_{p,q}$ does \emph{not} correspond to a valid Wasserstein gradient flow for the {\em Gaussian kernel}, similar to the finding of \citet{franz2026drifting} for the KL-KDE drift with {\em non-Gaussian kernels}   (Section \ref{sec:kl_drift}).

\begin{restatable}[Non-gradient structure]{prop}{propCurl}\label{thm:curl}
The velocity field $V^{\mathrm{SP}}_{p,q}$ in \eqref{eq:proxy_score}, with a Gaussian kernel, corresponds to a valid Wasserstein gradient flow only when the gradients of the pre-conditioner and the score difference are linearly dependent everywhere:
\begin{equation}\label{eq:curlCondition}
    \partial_i Z(x) \partial_j \Phi(x) = \partial_j Z(x) \partial_i \Phi(x) \quad \text{for all } i,j,
\end{equation}
where $\Phi(x) = \frac{\tau}{2}(\log(k * \tilde{p})(x) - \log(k * \tilde{q})(x))$. Equivalently, the wedge product must vanish: $\nabla Z \wedge \nabla \Phi = 0$.
\end{restatable}
We give here a simple example where \Cref{eq:curlCondition} fails.

\paragraph{A toy example.}
Let $d=2$ and take $p=\delta_b, q=\delta_a,
a=(-1,0), b=(1,0)$  and $k(x,y)=\exp\{-\|x-y\|^2/\tau\}$. Then, one can easily check that 
\begin{equation}
\Phi(x)
=
\frac{\tau}{2}
\left[
\log(k*\widetilde p)(x)-\log(k*\widetilde q)(x)
\right]
=
1+2x_1.
\end{equation}
On the other hand,
\begin{equation}
Z(x)
=
s^+(x)s^-(x)
=
\exp(2/\tau)\sqrt{k(x,b)k(x,a)}.
\end{equation}
So at  $x=(0,1)$, simple calculations yield $\nabla \Phi(0,1)=(2,0)$ and $\nabla Z(0,1)=
\left(0,-2Z(0,1)/\tau \right)$ which are not parallel, or equivalently
\begin{equation}
\partial_1 Z(0,1)\partial_2\Phi(0,1)
-
\partial_2 Z(0,1)\partial_1\Phi(0,1)
=
\frac{4Z(0,1)}{\tau}
\neq 0.
\end{equation}
Consequently, the necessary condition in \Cref{thm:curl} fails.

\subsection{Illustration: Sinkhorn Proxy flow vs KL flow and Wasserstein-2 flow}

To illustrate a key limitation of both KL-KDE and the Sinkhorn Proxy,
we demonstrate that neither is able to shift mass from one mode to another when the modes
are far apart relative to kernel width $\tau$. This  indicates
that the Sinkhorn Proxy does not recover the behaviour one would expect for gradient flow on a mass transport metric. 
 
 We consider a synthetic example where the target distribution $p$ and the model distribution $q$ are mixtures of two identical point masses at $+D$ and $-D$, differing only in their mixing proportions $\alpha$ and $\beta$, with $\alpha, \beta \in (0,1)$ and $\alpha \neq \beta$:
\begin{align}
    p(y) &= \alpha \updelta_{-D}(y) + (1-\alpha) \updelta_{+D}(y), \\
    q(x) &= \beta \updelta_{-D}(x) + (1-\beta) \updelta_{+D}(x).
\end{align}
We assume $\beta < \alpha$, meaning the model has insufficient mass at the left mode ($-D$) and excess mass at the right mode ($+D$). 

We evaluate the continuous vector fields for a model particle currently residing within the overcrowded cluster at $x = +D$. We use the unnormalized Gaussian kernel $k(x,y) = \exp(-\|x-y\|^2/\tau)$. 
We require the kernel bandwidth $\tau$ to be sufficiently narrow compared to the mode separation $2D$, such that
\begin{equation}\label{eq:thinKernel}
    \epsilon = k(D, -D) = \exp\left(-\frac{4D^2}{\tau}\right) \ll 1.
\end{equation}
For the algorithm to successfully correct the mixing proportions, it must yield a strict \textit{negative} drift at $x = +D$, routing excess mass across the manifold to the starved mode at $-D$.

The proofs of all the results in this Section are given in Appendix~\ref{sec:failureSyntheticProofs}. We first describe the failure of KL-KDE.
\begin{restatable}{prop}{KLfailuremode}
For fixed $D>0$ and fixed $\alpha,\beta\in(0,1)$ with $\beta<\alpha$, as $\tau\downarrow0$ (equivalently $\epsilon=\exp(-4D^2/\tau)\downarrow0$ and \eqref{eq:thinKernel} holds) the velocity field $V_{\textup{KL}}$ of the KL-KDE flow with Gaussian kernel at $x = +D$ is given by
\begin{equation}
    V^{\textup{KL}}_{p,q}(D) = -2D \epsilon \left( \frac{\alpha}{1-\alpha} - \frac{\beta}{1-\beta} \right)  + o(\epsilon).
\end{equation}
\end{restatable}
Unsurprisingly, the velocity vanishes for small $\epsilon$.
We now show that the Sinkhorn proxy method has the  same failure mode.
\begin{restatable}{prop}{propFailureMode}
For fixed $D>0$ and fixed $\alpha,\beta\in(0,1)$ with $\beta<\alpha$, as $\tau\downarrow0$ (equivalently $\epsilon=\exp(-4D^2/\tau)\downarrow0$),
 the velocity field of the Sinkhorn Proxy flow with Gaussian kernel at $x =D$ is given by
\begin{equation}
     V^{\mathrm{SP}}_{p,q}(D)
     = -2D\epsilon 
     \underbrace{\left[ \sqrt{\frac{1-\alpha}{\beta}}   \right]}_{\text{(*)}}
     \left( \frac{\alpha}{1-\alpha} - \frac{\beta}{1-\beta} \right) + o(\epsilon),
\end{equation}
where the term (*) represents the difference with the KL-KDE flow.
\end{restatable}

We now contrast the Sinkhorn proxy velocity field with the mass transport dynamics of an optimal transport distance, specifically the barycentric velocity at $x = +D$ for the Wasserstein gradient flow of $\frac12 W_2^2(q,p)$ (see Appendix~\ref{sec:failureSyntheticProofs} for discussion).

\begin{restatable}{prop}{trueWassersteinGF}
The barycentric velocity of the Wasserstein gradient flow on  $\frac12 W_2^2(q,p)$ at $x=D$  is given by
\begin{equation}
    V^{W_2}_{p,q}(D) = -2D \left( \frac{\alpha - \beta}{1 - \beta} \right).
\end{equation}
\end{restatable}

Velocity increases with D, and there is no scaling by an infinitesimal $\epsilon$.

\section{Extensions}\label{sec:extensions}
WGF techniques have been previously proposed to perform generative modeling for various functionals; e.g., Maximum Mean Discrepancy \citep{arbel2019mmdflow}, Sobolev witness \citep{mroueh2019sobolev} and Sliced-Wasserstein distance \citep{liutkus2019sliced,cozzi2024swflow}. WGF have also been used to solve inverse problems \citep{crucinio2022solving}. Contrary to GMD, all these approaches implement time-discretization of the ODE \eqref{eq:particle-ode}.  As noted by \citet{cao2026gradientflowdrifting,turan2026generativedriftingsecretlyscore}, the methodology of \cite{deng2026drifting} offers an alternative that can be readily implemented for any Wassertein gradient flow. This might prove a useful addition to existing techniques.

\subsection{Practical algorithm}\label{sec:algo}
Given an energy function $\cF_p(q)$ such that  $\cF_p(p)=0$, the general method to develop a one-stage generator $f_\theta$ to sample from $p$ is as follows:
\begin{itemize}
\item identify its first variation $\delta \cF_p/\delta q(x)$;
\item set $V_{p,q}(x)=-\nabla \delta \cF_p/\delta q(x)$;
\item approximate $V_{p,q_\theta}(x)$ via $\widehat{V}_{p,q_\theta}(x)$ (e.g. expectations are approximated by mini-batches, approximation of Tweedie's formula etc.). When empirical samples $\{y_j\}\iidsim p$ and $x_k=f_\theta(\epsilon_k)$, $\{\epsilon_k\}\iidsim \mu$ are available, we write $\widehat V_{p,q_{\theta}}(x;\{y_j\}_{j=1}^{N^+},\{x_k\}_{k=1}^{N^-})$.
\end{itemize}
Finally instead of minimizing $\mathbb{E}_{\epsilon \sim \mu}[\|\widehat{V}_{p,q_\theta}(f_\theta(\epsilon))\|^2]$ w.r.t. $\theta$, minimize
\begin{equation}\label{eq:general_loss}
\mathcal{L}(\theta)=\E_{\epsilon\sim \mu}\Big[\|f_{\theta}(\epsilon)-\sg\!\big(f_{\theta}(\epsilon) + \eta\widehat{V}_{p,q_{\theta}}(f_{\theta}(\epsilon))\big)\|^2\Big]
\end{equation}
where $\eta>0$ is a hyperparameter. 
The generic algorithm is detailed in Algorithm \ref{alg:WGF}.
\begin{algorithm}[t]
\caption{Generic drifted-target method to minimize an energy function $\cF_p(q)$}
\label{alg:WGF}
\begin{algorithmic}[1]
\Require data sampler $y\sim p$, noise sampler $\epsilon\sim \mu$, generator $f_\theta$, initial parameter $\theta$
\For{$\ell=0,1,2,\dots$}
  \State sample $\{y_j\}_{j=1}^{N^+}\iidsim p$ \Comment{Real batch}
  \State sample $\{\epsilon_i\}_{i=1}^{N^-}\iidsim \mu$, set $x_i=f_{\theta}(\epsilon_i)$ for all $i \in \{1,\ldots,N^-\}$ \Comment{Model batch}
  \State compute drifts $\widehat V_{p,q_{\theta}}(x_i) = \widehat V_{p,q_{\theta}}(x_i;\{y_j\}_{j=1}^{N^+},\{x_k\}_{k=1}^{N^-})$
  \State update $\theta$ by minimizing $\sum_{i=1}^{N^-} \|f_\theta(\epsilon_i)-\sg(x_i+ \widehat{V}_{p,q_\theta}(x_i))\|^2$
\EndFor
\end{algorithmic}
\end{algorithm}

The following Proposition, proven in Appendix~\ref{sec:loss_gradient_proof}, should be read as a surrogate-gradient identity: the target is detached at the current parameter value. A gradient step on the drifted target loss evaluated at the current iterate is equal, up to the scalar $2\eta$, to a gradient step on
the energy functional. Therefore, we set the $\eta=1$ in Algorithm~\ref{alg:WGF}.
\begin{restatable}[Surrogate-gradient identity]{prop}{propLossGradient}\label{prop:loss_gradient}
Let $q_\theta = (f_\theta)_\# \mu$ be the pushforward of $\mu$ through $f_\theta$, and suppose we have the exact velocity field $V_{p,q_\theta}(x) = -\nabla_x \frac{\delta \cF_p}{\delta q_\theta}(x)$ for some functional $\mathcal{F}_p$. Then
\begin{equation}
    \nabla_\theta \mathcal{L}(\theta) = 2\eta\, \nabla_\theta \cF_p(q_\theta).
\end{equation}
\end{restatable}

\subsection{Applications}\label{sec:otherDivergenceMeasures}

The approximation of $V_{p,q_\theta}$ is very dependent on the energy functional being minimized.
We give a few potential applications using different types of flows, in addition to the KL flow in Section~\ref{sec:kl_flow}.

\paragraph{Maximum mean discrepancy (MMD) flow} 
Let $k:\R^d\times\R^d\to\R$ be a differentiable reproducing kernel (e.g.\ RBF), and let $\mathcal{H}$ be the corresponding reproducing kernel Hilbert space. The Maximum Mean Discrepancy (MMD) \citep{mmd-jmlr} is defined as
\begin{align}
\mathrm{MMD}_{k}^{2}(p,q) &=\left(\sup_{g\in\mathcal{H},\|g\|\le1}\E_{x\sim q}[g(x)]-\E_{y\sim p}[g(y)]\right)^{2} \\
  &=\E_{x,x'\sim q}[k(x,x')]-2\E_{x\sim q,y\sim p}[k(x,y)]+\E_{y,y'\sim p}[k(y,y')].
\end{align} 
The function $g^*$ attaining the supremum is the MMD witness function. 
The WGF applied to $\cF^{\mathrm{MMD}}_p(q)=\frac12\,\mathrm{MMD}_k^2(p,q)$ was first studied by \citet{arbel2019mmdflow} who showed that the first variation up to normalization is

\begin{equation}\label{eq:mmd-variation}
\frac{\delta \cF^{\mathrm{MMD}}_p}{\delta q}(x)=
g^*(x):=\E_{x' \sim q}[k(x,x')]-\E_{y\sim p}[k(x,y)],
\end{equation}
so
\begin{equation}\label{eq:mmd-drift}
V^{\mathrm{MMD}}_{p,q}(x)
=
-\E_{x' \sim q}[\nabla_x k(x,x')]
+\E_{y \sim p}[\nabla_x k(x,y)].
\end{equation}
This can be approximated by the mini-batch estimator
\begin{equation}\label{eq:mmd-mc}
\widehat V_{p, q}^{\mathrm{MMD}}(x)
=\frac{1}{N^+}\sum_{i=1}^{N^+} \nabla_x k(x,y_i)-\frac{1}{N^-}\sum_{j=1}^{N^-} \nabla_x k(x,x_j).
\end{equation}
which recovers the MMD drifting approach proposed by \citet[][Appendix C.2]{deng2026drifting}, as demonstrated by \citet[][Section 4.5]{cao2026gradientflowdrifting}.
The meaning of ``kernel''  is fundamentally different to the KL case: here, it is a reproducing kernel (a positive definite function, which need not be a Parzen window), informing the smoothness of the witness function (first variation). When a characteristic kernel (which includes Gauss and Laplace kernels) is used, the MMD is zero if and only if $p=q$ \citep{SriGreFukLanetal10,SriFukLan11}. 

As a practical matter, Wasserstein gradient flow on the MMD has much better convergence behaviour when the kernel is chosen adaptively, according to the current location of the $q$ particles relative to those from $p$ \citep{galashov2025deep,chen2025deregularized}. As an alternative to adapting the kernel, a sum of kernels over different bandwidths can be used, as was done by \citeauthor{deng2026drifting} \citep[see][Section 5, for the multiple kernel learning perspective]{Cortesetal2009}.

A related perspective is the use of the MMD as a distributional loss (critic) in GAN training  \citet{li2015generative,dziugaite2015training}, and the related diffusion model of \citet{zhou25inductive}. The present discussion makes  explicit the connection with Wasserstein gradient flows.

\paragraph{Sliced-Wasserstein (SW) flow.}
Let $\mathbb{S}^{d-1}$ be the unit hypersphere in $\R^d$, and for any direction $\vartheta\in\mathbb{S}^{d-1}$, let $\pi_\vartheta(x)=x\cdot\vartheta$ denote the 1D projection of a point $x$ onto $\vartheta$. We define the squared Sliced-Wasserstein distance as
\[
\mathrm{SW}_2^2(q,p)=\int_{\mathbb{S}^{d-1}} W_2^2\big((\pi_\vartheta)_{\#} q,\;(\pi_\vartheta)_{\#} p\big)\,\dd\vartheta,
\]
where $(\pi_\vartheta)\# q$ denotes the push-forward (i.e., the 1D marginal) of the distribution $q$ along the direction $\vartheta$, and $\dd\vartheta$ represents the uniform probability measure on the sphere. We set the energy function to $\cF^{\mathrm{SW}}_p(q)\;\triangleq\;\frac12\,\mathrm{SW}_2^2(q,p)$.
\citet{cozzi2024swflow} showed that 
\begin{equation}\label{eq:sw-variation}
\frac{\delta \cF^{\mathrm{SW}}_p}{\delta q}(x)
=
\int_{\mathbb{S}^{d-1}}\varphi_\vartheta(\vartheta\cdot x)\,\dd\vartheta,\quad V^{\mathrm{SW}}_{p,q}(x)
=
-\nabla_x \frac{\delta \cF^{\mathrm{SW}}_p}{\delta q}(x)
=
\int_{\mathbb{S}^{d-1}}\big(T_\vartheta(\vartheta\cdot x)-\vartheta\cdot x\big)\,\vartheta\;\dd\vartheta,
\end{equation}
where $\varphi_\vartheta$ is the 1D Kantorovich potential and $T_\vartheta:\R \to \R$ is the 1D optimal transport map pushing $(\pi_\vartheta)_{\#} q$ to $(\pi_\vartheta)_{\#} p$. This can be  approximated using a mini-batch of real data samples and model samples, and $V^{\mathrm{SW}}_{p,q}(x)$ then approximated by uniformly sampled directions $(\vartheta^{i})_{i=1}^L$ on $\mathbb{S}^{d-1}$.

\paragraph{Dual Forms of $f$-Divergences.}
The $f$-divergences (notably Jensen-Shannon) are the standard divergence measures used in GANs. As shown by \citet{nowozin2016fgan}, critics based on $f$-divergences can be formulated using the Fenchel dual of their corresponding divergences. This dual formulation introduces an optimal witness function $g^*$ whose gradient is the Wasserstein gradient of the $f$-divergence. Building upon this, \citet{glaser2021kale} proposed the KL Approximate Lower-bound Estimator (KALE) as an alternative to the KL gradient flow discussed in Section \ref{sec:fixedpoint}. KALE operates by regularizing the dual KL divergence over a designated function class $\mathcal{H}$.
For a regularization parameter $\lambda > 0$, the estimator of \citet[][Proposition 2]{glaser2021kale} is:
\begin{equation}\label{eq:KALE}
    \cF^\text{KALE}_p(q)=  (1+\lambda) \max_{g \in \mathcal{H}} \left\{ \mathbb{E}_{y \sim q}[g(y)] - \mathbb{E}_{x \sim p}[e^{g(x)}] + 1 - \frac{\lambda}{2}\|g\|_{\mathcal{H}}^2 \right\}.
\end{equation}
This is computed via a quadratic program when $\mathcal{H}$ is a reproducing kernel Hilbert space \citep{glaser2021kale}, and by SGD on the weights $\phi$ of $g_{\phi}$, when $\mathcal{H}$ is a neural network \citep{nowozin2016fgan,galashov2025deep}.

As for the MMD, $\nabla \frac{\delta \cF_p}{\delta q}(x)= (1+\lambda) \nabla g^* $, where $g^* \in \mathcal{H}$ achieves the supremum in \eqref{eq:KALE}.
A neural network implementation of the KL gradient flow, which parameterizes the critic function with a deep architecture, is given by \citep[][Section 6 and Appendix D]{galashov2025deep}. A potential advantage of the Fenchel dual instantiation of the Wasserstein Gradient on the KL, over the flow described in Section \ref{sec:fixedpoint}, is that the former does not require kernel density estimates in its definition.  As an alternative to the KL, a gradient flow on the dual approximation to the $\chi^2$ divergence is given by \citet{chen2025deregularized}, which has closed-form updates, unlike the KALE. 

\subsection{Illustration}\label{sec:Illustration}
As a proof of concept, we ran experiments to train generators using 6 flow objectives above on 5 synthetic 2D datasets. 
We use $\tau$ to denote hyperparameters in different flow, and vary its value over a wide range to test the sensitivity and effectiveness of the different approaches.
For the KL and MMD flows, $\tau$ is the kernel bandwidth in \eqref{eq:rbf}.
For Sinkhorn-based flows, including \citet{deng2026drifting}'s Algorithm 2 (DA2) and the approximate Sinkhorn algorithm we derived, $\tau$ denotes the entropic regularization weight. The plain Sinkhorn flows used 100 Sinkhorn-Knopp iterations to evaluate the potentials. All Sinkhorn-based algorithms operate in log-probability space for numerical stability. 
For the SW flow, we varied the number of random slices $L$ from 1--100 but did not find any effect of $L$ on sample quality.
The generator network architecture is a feedforward ResNet, with parameters optimized by the Adam algorithm with learning rate $10^{-4}$, trained for $10^6$ steps. Each data-drift type combination is repeated 20 times with different random seeds for the dataset and parameter initializations.

We used the MMD$^2$ under a median-heuristic Gaussian kernel as the metric for evaluating sample quality at the end of training. The key results are shown in Figure~\ref{fig:mmd}.
Sliced-Wasserstein is omitted because its hyperparameter $L$ does not have any effect on sample quality.
All flows can attain similarly best performances at some
sweet spot of their relevant hyperparameter.
The Sinkhorn drift can tolerate a small $\tau$ and only gets worse for large $\tau$. 
Interestingly, our Sinkhorn Proxy has better tolerance to smaller values of $\tau$ compared to DA2.
Qualitatively, Figure~\ref{fig:moon_results} shows the true and the generated sample histograms for the Moons dataset. Histograms of other datasets are left in Appendix~\ref{sec:additional_results}. 
For KL, a too small noise level leads to mode-collapse common for score-based methods \citep{wenliang2020blindness}, and a too large noise level causes the learned distribution to ignore local noise structure (Figures~\ref{fig:moon_results}, \ref{fig:circles_results} and \ref{fig:gaussians_results}). The mode-collapse is not as pronounced for the Sinkhorn Proxy or DA2, which might be due to the ``hub-supression'' effect discussed in Appendix~\ref{sec:whyCloseToSinkhorn}.

\begin{figure}
    \centering
    \includegraphics[width=0.99\linewidth]{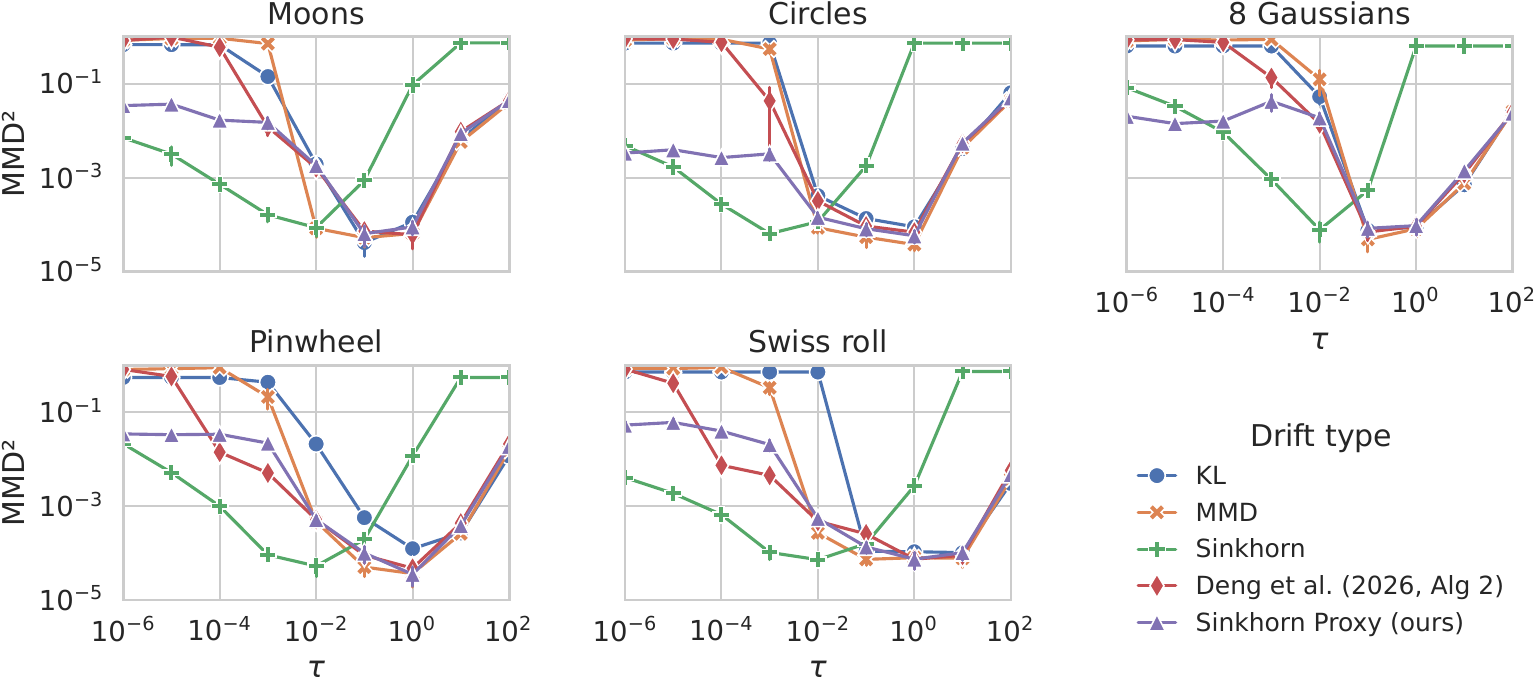}
    \caption{MMD between true and generated samples trained by different drift types.}
    \label{fig:mmd}
\end{figure}

\begin{figure}
    \centering
    \includegraphics[width=0.9\linewidth]{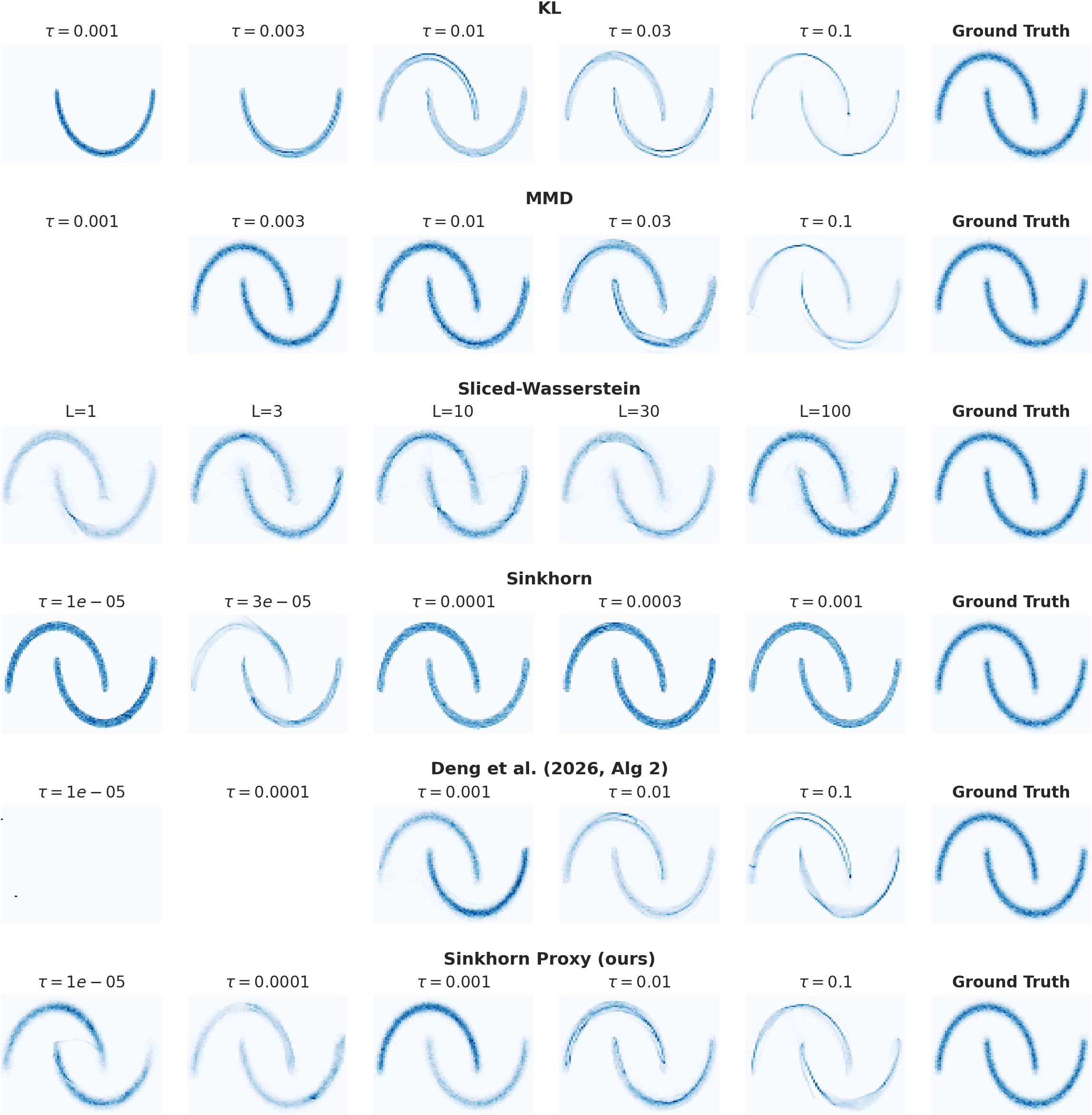}
    \caption{True and generated samples for different types of drift and hyperparameters. Empty panel means the samples have diverged.}
    \label{fig:moon_results}
\end{figure}

\section*{Acknowledgments}
The authors are grateful to Guillaume Couairon and Christopher Marouani for their helpful comments.

\bibliographystyle{apalike}
\bibliography{refs}

\appendix

\section{Wasserstein Gradient flow of KL on Parzen Window density estimates}\label{sec:parzen_smoothed_kl}

We derive the first variation and the Wasserstein gradient flow velocity for the KDE-smoothed forward KL divergence $\cF^{(\tau)}_p(q):=\textup{KL}(q_\tau \| p_\tau)$, where $q_\tau = k_\tau * q$ and $p_\tau = k_\tau * p$ are Parzen window density estimates with a symmetric kernel $k_\tau$.

The functional $\cF^{(\tau)}_p$ is a composition: $q$ first maps to $q_\tau$ via the linear convolution $q \mapsto k_\tau * q$, and then the standard KL divergence is evaluated. By \eqref{eq:firstvariationKL}, the first variation of $\textup{KL}(\cdot \| p_\tau)$ with respect to its first argument is
\begin{equation}
    \frac{\delta\, \textup{KL}(\cdot \| p_\tau)}{\delta q_\tau}(x) = \log q_\tau(x) - \log p_\tau(x) + 1.
\end{equation}
To obtain the first variation of $\cF^{(\tau)}_p$ with respect to $q$, we apply the chain rule through the convolution. For any perturbation $\chi$ with $\int \chi(y)\,\dd y = 0$, linearity gives
$k_\tau * (q + h\chi) = k_\tau * q + h\,(k_\tau * \chi) = q_\tau + h\,(k_\tau * \chi)$,
so
\begin{align}
    \left.\frac{\dd}{\dd h}\right|_{h=0} \cF^{(\tau)}_p(q + h\chi)
    &= \left.\frac{\dd}{\dd h}\right|_{h=0} \textup{KL}\!\big(q_\tau + h\,(k_\tau * \chi) \,\big\|\, p_\tau\big) \nonumber\\
    &= \int \frac{\delta\, \textup{KL}(\cdot \| p_\tau)}{\delta q_\tau}(x)\; (k_\tau * \chi)(x)\,\dd x \nonumber\\
    &= \int \Big(\log q_\tau(x) - \log p_\tau(x) + 1\Big) \int k_\tau(x,y)\,\chi(y)\,\dd y\;\dd x,
\end{align}
where the second line is the definition of first variation applied to $\textup{KL}(\cdot \| p_\tau)$ with perturbation $k_\tau * \chi$.
Exchanging the order of integration (Fubini) and using the symmetry $k_\tau(x,y) = k_\tau(y,x)$, this equals
\begin{equation}
    \int \underbrace{\left[\int k_\tau(y,x)\Big(\log q_\tau(x) - \log p_\tau(x) + 1\Big)\,\dd x\right]}_{\displaystyle = \;\frac{\delta \cF^{(\tau)}_p}{\delta q}(y)} \chi(y)\,\dd y.
\end{equation}
Because $\int k_\tau(y,x)\,\dd x$ is constant in $y$ for translation-invariant kernels, the $+1$ term contributes only a constant and vanishes under $\nabla_y$. The WGF velocity of the KDE-smoothed KL divergence is therefore
\begin{equation}\label{eq:parzen_kl_velocity}
     V^{(\tau)}_{p,q}(y) = -\nabla_y \frac{\delta \cF^{(\tau)}_p}{\delta q}(y) = -\nabla_y \int k_\tau(y,x) \log \frac{q_\tau(x)}{p_\tau(x)}\, \dd x.
\end{equation}

\section{Details of the Sinkhorn Proxy}\label{sec:sinkhornProxyDetails}

This appendix contains the full derivation and analysis of the Sinkhorn Proxy drift field summarized in \Cref{sec:sinkhorn}. We will switch from the generic notation $V_{p,q}$ to method-specific notation, such as $V^{\mathrm{KL}}_{p,q}$, $V^{\mathrm{SP}}_{p,q}$, in the following analysis to avoid ambiguity.

\subsection{Drifting Models as Fixed Point of a Wasserstein Gradient Flow on an Approximate Sinkhorn Divergence}\label{sec:sinkhornProxyDetailedIntro}

As noted in their paper, the drifting model as implemented does not use \citep[eq. (8)]{deng2026drifting} in practice, but the approach described in \citep[Algorithm 2, Appendix A]{deng2026drifting}. We show that this represents a fixed point of a probability flow on an approximation to the Sinkhorn Optimal Transport distance, where the velocity field is zero everywhere if and only if $p$ and $q$ coincide under the stated assumptions. Despite this resemblance to a Sinkhorn distance, however, the resulting velocity field is not generally conservative and therefore cannot, for generic pairs $(p,q)$ in dimension at least two, be the Wasserstein gradient of a scalar distributional loss. Its transport dynamics are in certain important respects closer to a gradient flow on the KL.
 
We begin with a brief review of the Sinkhorn distance and its associated Wasserstein gradient flow on a set of particles. Let $q,p\in \cP_2(\R^d)$ be probability measures with bounded second moment and let $C(x,y) = \tfrac{1}{2}\|x-y\|^2$. The \emph{debiased Sinkhorn divergence} \citep{Ramdas17wassersteinTest} is defined as:
\begin{equation}\label{eq:debiasedSinkhorn}
    S_\tau(q, p) = \text{OT}_\tau(q, p) - \frac{1}{2}\text{OT}_\tau(q, q) - \frac{1}{2}\text{OT}_\tau(p, p),
\end{equation}
where $\text{OT}_\tau(q, p)$ is the entropy-regularized OT distance \citep{cuturi2013sinkhorn}, and $\tau > 0$ is the entropic regularization parameter.

\propSinkhornVelocity*
\begin{proof}
By the envelope theorem \citep[][Proposition~1]{feydy2019sinkhorn}, the first variation of $S_\tau$ with respect to its first argument is the difference of optimal dual potentials,
\begin{equation}
    \frac{\delta S_\tau(q, p)}{\delta q}(x) = g_{q,p}(x) - g_{q,q}(x),
\end{equation}
where $g_{q,p}$ and $g_{q,q}$ are the optimal potentials on the $q$-side of $\textup{OT}_\tau(q,p)$ and $\textup{OT}_\tau(q,q)$, respectively. For the self term $\textup{OT}_\tau(q,q)$, the variation of $q$ affects both marginals; the derivative of $\frac12\textup{OT}_\tau(q,q)$ is $\frac12(f_{q,q}+g_{q,q})$. Since the two marginals and the cost are identical, the optimal potentials may be chosen with $f_{q,q}=g_{q,q}$ up to an additive constant, so the self-term contribution is $g_{q,q}$. The WGF velocity is therefore
\begin{equation}
    \widehat{V}^\textup{S}_{p,q}(x) = -\nabla_x \frac{\delta S_\tau(q,p)}{\delta q}(x) = -\nabla_x g_{q,p}(x) + \nabla_x g_{q,q}(x).
\end{equation}
It remains to compute these gradients. For the $g_{q,p}$ term, the joint optimal transport plan $\pi^+$ and the dual potentials $g_{q,p}$, $h_{q,p}$ are related by \citep[][Equation~9]{feydy2019sinkhorn}
\begin{align}
    \pi^+_{ij} = \frac{1}{N^+ N^-}\exp\left(\frac{g_{q,p}(x_i) + h_{q,p}(y_j) - \tfrac{1}{2}\|x_i - y_j\|^2}{\tau}\right).
\end{align}
The plan $\pi^+$ satisfies $\sum_j \pi^+_{ij} = \frac{1}{N^-}$ and $\sum_i \pi^+_{ij} = \frac{1}{N^+}$.
The potentials satisfy the coupled fixed-point equations:
\begin{align}
    g_{q,p}(x_i) &= -\tau \log \sum_{j=1}^{N^+} \frac{1}{N^+}\exp\left(\frac{h_{q,p}(y_j) - \tfrac{1}{2}\|x_i - y_j\|^2}{\tau}\right), \\
    h_{q,p}(y_j) &= -\tau \log \sum_{i=1}^{N^-} \frac{1}{N^-}\exp\left(\frac{g_{q,p}(x_i) - \tfrac{1}{2}\|x_i - y_j\|^2}{\tau}\right).
\end{align}
Taking $\nabla_{x_i}$ of the first equation, the $\frac{1}{N^+}$ factors appear in both numerator and denominator and cancel:
\begin{align}
    \nabla_{x_i} g_{q,p}(x_i)
    \;=\; -\tau \frac{\sum_{j=1}^{N^+}\frac{1}{N^+}\exp\Big(\frac{h_{q,p}(y_j) - \tfrac{1}{2}\|x_i - y_j\|^2}{\tau}\Big) \frac{-(x_i - y_j)}{\tau}}{\sum_{j'=1}^{N^+} \frac{1}{N^+}\exp\Big(\frac{h_{q,p}(y_{j'}) - \tfrac{1}{2}\|x_i - y_{j'}\|^2}{\tau}\Big)}
    \;=\; \sum_{j=1}^{N^+} W^+_{ij} (x_i - y_j),
\end{align}
where $W^+_{ij} \triangleq \pi^+_{ij}/\sum_{j'}\pi^+_{ij'} = N^- \pi^+_{ij}$ is the row-stochastic \emph{conditional} plan (with $\sum_j W^+_{ij} = 1$), and the last equality holds because the softmax ratio equals $W^+_{ij}$.
An identical calculation for the self-term gives $\nabla_{x_i} g_{q,q}(x_i) = \sum_k W^-_{ik}(x_i - x_k)$. Substituting back:
\begin{equation}\label{eq:expanded}
    \widehat{V}^\textup{S}_{p,q}(x_i) = -\sum_{j=1}^{N^+} W^+_{ij}(x_i - y_j) + \sum_{k=1}^{N^-} W^-_{ik}(x_i - x_k)  = \sum_{j=1}^{N^+} W^+_{ij}\, y_j - \sum_{k=1}^{N^-} W^-_{ik}\, x_k,
\end{equation}
since $W^+$ and $W^-$ both have unit row sums ($\sum_j W^+_{ij} = \sum_k W^-_{ik} = 1$).
\end{proof}

\subsection{The One-Shot Sinkhorn Proxy}\label{sec:oneShot} 
Equation~\eqref{eq:sinkhorn_exact} requires computing the true conditional optimal transport plans (i.e., running Sinkhorn to convergence) at every timestep. 
A drifting model using the fully converged Sinkhorn method was proposed by \cite{he2026sinkhorndrifting} and \citet[][Section 6.3, Appendix I.4]{turan2026generativedriftingsecretlyscore}.
To recover \cite[][Algorithm 2, Appendix A]{deng2026drifting} to a close approximation, we replace the exact solution with a one-shot approximation. 

Let $z^+_{im} = -C^+_{im}/\tau$ and $z^-_{im} = -C^-_{im}/\tau$ be the scaled affinities, where $i$ is an index of a $q$ sample, and $m$ may be either a $p$-sample index or a $q$-sample index, depending on whether we refer to $C^+_{ij}$ or $C^-_{ij}$, respectively.  We will focus on the case of the squared Euclidean distance,\footnote{Note that the exact Sinkhorn flow derived in Proposition~\ref{prop:sinkhorn_velocity} uses the standard optimal transport ground cost $C(x,y) = \frac{1}{2}\|x-y\|^2$. The empirical algorithm's use of the unscaled distance $\|x-y\|^2$  introduces a constant factor of $2$ into the  spatial gradient, which in practice is  absorbed by the step-size hyperparameter $\eta$ during optimization.} $C^+_{im}=\|x_i-y_m\|^2$.
We approximate the Sinkhorn solution via the geometric mean of the row-wise and column-wise softmax normalizations (with additional $N^+$ and $N^-$ scaling to ensure the population limits),\footnote{In the original paper, the  expression was
\begin{equation}
 {A}^+_{im} = \sqrt{\left(\frac{\exp(z_{im})}{\sum^{N_+ +N_-}_p \exp(z_{ip})}\right) \left(\frac{\exp(z_{im})}{\sum^{N_-}_q \exp(z_{qm})}\right)}.
\end{equation} 
We replace the sum over $N_+ +N_-$ samples with a sum over only one set of samples, for reasons that will become clear.} 
\begin{equation}
 \begin{aligned}
    {A}^+_{im} &= \sqrt{\frac{\exp(z^+_{im})}{(N^+)^{-1}\sum_{u=1}^{N^+} \exp(z^+_{iu})} \cdot \frac{\exp(z^+_{im})}{(N^-)^{-1}\sum_{v=1}^{N^-} \exp(z^+_{vm})}}\;,
    \\
    {A}^-_{im} &= \sqrt{\frac{\exp(z^-_{im})}{(N^-)^{-1}\sum_{u=1}^{N^-} \exp(z^-_{iu})} \cdot \frac{\exp(z^-_{im})}{(N^-)^{-1}\sum_{v=1}^{N^-} \exp(z^-_{vm})}}\;,
    \end{aligned}
\end{equation}
where we again use ${A}^+_{im}$ when $m$ is a $p$-sample index, and ${A}^-_{im}$ for a $q$-sample index $m$.  This serves as a pseudo-plan, used instead of the fully converged Sinkhorn solution. The relation of this approximation to the full Sinkhorn algorithm is described in Appendix \ref{sec:whyCloseToSinkhorn}.  Setting aside the scaling by $N^+$ and $N^-$, we emphasize that unlike the full Sinkhorn solution, the rows and columns of the unscaled ${A}$ are no longer guaranteed to sum to 1, and the use of the envelope theorem \citep{feydy2019sinkhorn} in obtaining the spatial gradient is no longer valid. 

Because  ${A}_{im}$ does not perfectly enforce the normalization conditions, we may not directly use ${A}^+$ and ${A}^-$  to compute the velocity field. This is addressed by a cross-weighting trick. Define
\begin{equation}
    s^+_i =  \frac{1}{N^+}  \sum_j  A^+_{ij} \quad \text{and} \quad s^-_i =  \frac{1}{N^-} \sum_k {A}^-_{ik}
\end{equation}
Writing 
\begin{equation}
\begin{aligned}
    \widetilde{W}^+_{ij} &=
    \frac{1}{N^+} {A}^+_{ij} \cdot s^-_i \\
    &=
 \frac{1}{N^+}       \left( \frac{{A}^+_{ij}}{ s_i^+ } \right)
      \times \Big( s^+_i \cdot s^-_i \Big)    \\
    &=
    \underbrace{\left( \frac{{A}^+_{ij}}{ \sum_j  A^+_{ij} } \right)}_{\text{conditional plan } (\tilde\pi^+_{ij})} \times \underbrace{\Big( s^+_i \cdot s^-_i \Big)}_{\text{scaling} (Z_i)}
\end{aligned}
\end{equation}
and defining $Z_i = s^+_i s^-_i$, then
\begin{equation}
 \begin{aligned}
    \widehat{V}^{\mathrm{SP}}_{p,q}(x_i) &= 
    - \sum_{j=1}^{N^+} \widetilde{W}^+_{ij} (x_i -  y_j) + \sum_{k=1}^{N^-} \widetilde{W}^-_{ik} (x_i - x_k ) \\
    &= - s^-_i \left[ \frac{1}{N^+} \sum_{j=1}^{N^+} A^+_{ij} (x_i - y_j) \right] + s^+_i \left[ \frac{1}{N^-} \sum_{k=1}^{N^-} A^-_{ik} (x_i - x_k) \right] \\
    &=  Z_i \left( \sum_j \tilde\pi^+_{ij} y_j - \sum_k \tilde\pi^-_{ik} x_k \right) 
 \end{aligned}
\end{equation}
which resembles \eqref{eq:expanded}, scaled by a particle-specific $Z_i$.
The drift vector $\widetilde{V}_i$ corresponds to the probability flow of the approximate Sinkhorn solution.  We emphasize  that due to the cross-weighting, the cancellation of $x_i$ in \eqref{eq:expanded} still holds.

\subsection{Alternative form, Sinkhorn proxy velocity field}\label{sec:population_counterpart}

An alternative form for the population velocity field is given in the following proposition.

\propScoreInterpretation*

\begin{proof}

Recall the original expression of the drift vector from \eqref{eq:initialSinkhornProxyDrift},
\begin{equation}
    V(x) = s^-(x) \int A^+(x,y) y \, dp(y) - s^+(x) \int A^-(x,x') x' \, dq(x'), \label{eq:drift_base}
\end{equation}
where
\begin{equation}
A^+(x,y) = \frac{k(x,y)}{\sqrt{p_\tau(x) q_\tau(y)}},  \qquad
    A^-(x,x') = \frac{k(x,x')}{\sqrt{q_\tau(x) q_\tau(x')}}. \label{eq:A_plus_minus}
\end{equation}
and
\[    s^+(x) = \int A^+(x,y) \, dp(y), \qquad
    s^-(x) = \int A^-(x,x') \, dq(x').
\]

Using the  definitions $d\tilde{p}(y) = q^{-1/2}_\tau(y) dp(y)$ and $d\tilde{q}(x) = q^{-1/2}_\tau(x) dq(x)$, we can express the row sums as kernel convolutions over these modified measures:
\begin{align}
    s^+(x) &= \frac{1}{\sqrt{p_\tau(x)}} \int \frac{ k(x,y)}{\sqrt{q_\tau(y)}} \, dp(y) = \frac{1}{\sqrt{p_\tau(x)}} (k * \tilde{p})(x), \label{eq:s_plus} \\
    s^-(x) &= \frac{1}{\sqrt{q_\tau(x)}} \int  \frac{k(x,x')}{\sqrt{q_\tau(x')}} \, dq(x') = \frac{1}{\sqrt{q_\tau(x)}} (k * \tilde{q})(x). \label{eq:s_minus}
\end{align}

Next, we replace the proxy affinities $A^+$ and $A^-$ inside the expected destination integrals in \eqref{eq:drift_base}, yielding:
\begin{equation}
    V(x) = \frac{s^-(x)}{\sqrt{p_\tau(x)}} \int y k(x,y) d\tilde{p}(y)  - \frac{s^+(x)}{\sqrt{q_\tau(x)}} \int x' k(x,x') d\tilde{q}(x') . \label{eq:drift_sub}
\end{equation}

To evaluate these spatial integrals, we use Tweedie's formula for the unnormalized Gaussian kernel $k(x,z) = \exp(-\|x-z\|^2/\tau)$, which for a measure $\mu(z)$, yields
\[
\frac{\int z k(x,z) d\mu(z) }{\int k(x,z) d\mu(z) } = x + \frac{\tau}{2} \nabla_x \log(k * \mu)(x).
\]
Multiplying both sides by the denominator $\int k(x,z) d\mu(z)  = (k * \mu)(x)$ yields:
\[
\int z k(x,z) d\mu(z)  = (k * \mu)(x) \left[ x + \frac{\tau}{2} \nabla_x \log(k * \mu)(x) \right].
\]

We apply this result to  \eqref{eq:drift_sub} for both modified measures $\tilde{p}$ and $\tilde{q}$. From our row sum derivations in \eqref{eq:s_plus} and \eqref{eq:s_minus}, we have $(k * \tilde{p})(x) = \sqrt{p_\tau(x)} s^+(x)$ and $(k * \tilde{q})(x) = \sqrt{q_\tau(x)} s^-(x)$. Substituting these relations into the right-hand side of Tweedie's formula gives:
\begin{align*}
    \int y k(x,y) d\tilde{p}(y)  &= \sqrt{p_\tau(x)} s^+(x) \left[ x + \frac{\tau}{2} \nabla_x \log(k * \tilde{p})(x) \right], \\
    \int x' k(x,x') d\tilde{q}(x')  &= \sqrt{q_\tau(x)} s^-(x) \left[ x + \frac{\tau}{2} \nabla_x \log(k * \tilde{q})(x) \right].
\end{align*}

Finally, we substitute these evaluated integrals back into the drift vector expression \eqref{eq:drift_sub}:
\begin{align*}
    V(x) &= \frac{s^-(x)}{\sqrt{p_\tau(x)}} \Bigg( \sqrt{p_\tau(x)} s^+(x) \left[ x + \frac{\tau}{2} \nabla_x \log(k * \tilde{p})(x) \right] \Bigg) \\
    &\quad - \frac{s^+(x)}{\sqrt{q_\tau(x)}} \Bigg( \sqrt{q_\tau(x)} s^-(x) \left[ x + \frac{\tau}{2} \nabla_x \log(k * \tilde{q})(x) \right] \Bigg) \\
    &= 
 s^-(x)s^+(x) \left[ x + \frac{\tau}{2} \nabla_x \log(k * \tilde{p})(x) \right] - s^+(x)s^-(x) \left[ x + \frac{\tau}{2} \nabla_x \log(k * \tilde{q})(x) \right].
\end{align*}

Defining $Z(x) = s^+(x)s^-(x)$ and expanding the brackets, the  coordinates $x$ cancel:
\begin{align*}
    V(x) 
    &= \frac{\tau}{2} Z(x) \big( \nabla_x \log(k * \tilde{p})(x) - \nabla_x \log(k * \tilde{q})(x) \big).
\end{align*}
\end{proof}

\subsection{Consistency of the Sinkhorn Proxy}\label{sec:zeroIFFPisQ}
\thmConsistency*

\begin{proof}
    We analyze the condition under which the cross-weighted vector field $V_{p,q}^\mathrm{SP}(x)$ vanishes everywhere.
    \vspace{1em}
    \noindent \textbf{($\impliedby$)} Assume $p = q$. It follows that the modified measures are identical:
    \[
    d\tilde{p}(y) = q^{-1/2}_\tau(y) dp(y) =  q^{-1/2}_\tau(x) dq(x) = d\tilde{q}(y) \quad \forall y.
    \]
    Consequently, their convolutions with the kernel $k$ are identical, yielding $\nabla_x \log(k * \tilde{p})(x) = \nabla_x \log(k * \tilde{q})(x)$ for all $x$, and the velocity field is identically zero everywhere.
    
    \vspace{1em}
    \noindent \textbf{($\implies$)} Assume $V_{p,q}^\mathrm{SP}(x) = 0$ for all $x$. The pre-conditioner $Z(x)$ is strictly positive everywhere (as the Gaussian $k$ is supported on $\mathbb{R}^d$), hence we must have:
    \[
    \nabla_x \log(k * \tilde{p})(x) = \nabla_x \log(k * \tilde{q})(x) \quad \forall x.
    \]
    Integrating this spatial gradient implies that the kernel-smoothed measures are proportional up to a global constant $\lambda > 0$:
    \[
    (k * \tilde{p})(x) = \lambda (k * \tilde{q})(x) \quad \forall x.
    \]
    When $k$ is characteristic \citep{SriGreFukLanetal10} and translation invariant on $\mathbb{R}^d$, then the mean embedding map $\mu \mapsto k * \mu$ is injective for finite signed measures \citep[][Proposition 2 and Section 3.2]{SriFukLan11}. This property holds for the Gaussian kernel, since its Fourier transform is everywhere non-zero. Assuming that the modified measures $\tilde{p}$ and $\tilde{q}$ have finite mass, then injectivity applies. Therefore, if the smoothed measures are proportional, the underlying modified measures must themselves be  proportional:
    \[
  d  \tilde{p}(y) = \lambda  d \tilde{q}(y) \quad \forall y.
    \]
    Substituting the definitions of $\tilde{p}$ and $\tilde{q}$ yields:
    \[
   q^{-1/2}_\tau(y) dp(y) = \lambda q^{-1/2}_\tau(x) dq(x) 
    \]
    Because the kernel-smoothed model density is strictly positive ($q_\tau(y) > 0$), then
    \[
  d  p(y) = \lambda d q(y) \quad \forall y.
    \]
    Since $p$ and $q$ are both valid probability density functions representing the target data distribution and the model distribution, respectively, they must integrate to 1,
    \[
    \int p(y) \, dy = \lambda \int q(y) \, dy  \implies \lambda = 1.
    \]
\end{proof}

\subsection{The Sinkhorn Proxy Velocity does not generally correspond to a Wasserstein Gradient Flow}\label{sec:curlProof}
\propCurl*
\begin{proof}
To be a valid Wasserstein gradient flow, the associated velocity vector field $V^{\mathrm{SP}}_{p,q}(x)$ driving the particles must take the form of a pure gradient field (see \eqref{eq:continuity}). Thus, there must exist a scalar function (the first variation $\frac{\delta \mathcal{F}_p}{\delta q}(x)$) whose gradient produces this field. 

We start with the form in \cref{thm:scoreInterpretationDrifting}. Define the term inside the parenthesis as 
\[ \Phi(x) = \frac{\tau}{2} \Big( \log(k * \tilde{p})(x) - \log(k * \tilde{q})(x) \Big). \]
We can then rewrite the Sinkhorn Proxy drift concisely as:
\[ V^{\mathrm{SP}}_{p,q}(x) = Z(x) \nabla_x \Phi(x). \]

By Poincar\'e's Lemma, a continuously differentiable vector field on a simply connected domain is a pure gradient field if and only if its Jacobian matrix is symmetric. Let $V_i$ denote the $i$-th component of $V^{\mathrm{SP}}_{p,q}$. The symmetry condition $\partial_j V_i = \partial_i V_j$ implies:
\begin{align}
\partial_j \big(Z(x) \partial_i \Phi(x)\big) &= \partial_i \big(Z(x) \partial_j \Phi(x)\big) \\
\partial_j Z(x) \partial_i \Phi(x) + Z(x) \partial_j \partial_i \Phi(x) &= \partial_i Z(x) \partial_j \Phi(x) + Z(x) \partial_i \partial_j \Phi(x).
\end{align}
By Schwarz's theorem, the mixed partial derivatives of $\Phi$ are equal ($\partial_j \partial_i \Phi = \partial_i \partial_j \Phi$), so
\begin{equation}
\partial_j Z(x) \partial_i \Phi(x) = \partial_i Z(x) \partial_j \Phi(x).
\end{equation}
In other words, the gradient of the scalar pre-conditioner $\nabla Z(x)$ must be perfectly parallel to the gradient of the score difference $\nabla \Phi(x)$ at all points in space. Given their distinct definitions, this geometric condition will not hold in general.
\end{proof}

\section{Geometric mean as approximation to Sinkhorn solution}
\label{sec:whyCloseToSinkhorn}

We demonstrate the relation between ${A}^+$  and the optimal transport plan obtained via Sinkhorn iteration.

In entropy-regularized optimal transport \citep{cuturi2013sinkhorn}, the interaction between points is governed by a strictly positive Gibbs kernel matrix $K$. The algorithm defines the independent pairwise logits as $Z^+_{ij} = -C^+_{ij} / \tau$. Exponentiating these logits yields  the Gibbs kernel:
\begin{equation}
    K_{ij} = \exp(Z^+_{ij}) = \exp\left(-\frac{C^+_{ij}}{\tau}\right)
\end{equation}

The objective of the Sinkhorn algorithm is to find the strictly positive vectors $u$ and $v$ such that the resulting transport plan $P = \text{diag}(u) K \text{diag}(v)$ satisfies target marginal constraints.  To recover the Sinkhorn proxy algorithm, these marginals are assumed to be {\em uniform},
\[
( P \mathbf{1} )^\top  = \mathbf{1} ^\top P  = \mathbf{1}.
\]
The resulting optimal transport plan $P^*$ takes the form
\begin{equation} \label{eq:sinkhorn_struct}
    P^*_{ij} = u_i K_{ij} v_j
\end{equation}
 The standard iterative Sinkhorn algorithm finds $u$ and $v$ by repeatedly alternating between dividing by row sums and column sums until the marginals match the target distributions.

Assume first that Sinkhorn initializes the column scaling vector as a vector of ones: $v^{(0)} = \mathbf{1}$. The first update for the source scaling vector $u$ is computed by dividing the target marginals by the row sums of the current scaled kernel,
\begin{equation}\label{eq:sink1}
    u^{(1)}_i = \frac{1}{(K v^{(0)})_i} = \frac{1}{(K \mathbf{1})_i} = \frac{1}{r_i}
\end{equation}
Applying this scaling alone yields $u^{(1)}_i K_{ij} = K_{ij}/r_i$, which corresponds exactly to the \texttt{softmax\_row} operation. 
 Consider now the case that we begin the Sinkhorn iterations with ($u^{(0)} = \mathbf{1}$). We would then make the update
\begin{equation}\label{eq:sink2}
    v^{(1)}_j = \frac{1}{(K^T u^{(0)})_j} = \frac{1}{(K^T \mathbf{1})_j} = \frac{1}{c_j}.
\end{equation}

Turning now to the approach of \citet[][Algorithm 2]{deng2026drifting}, we  examine the relation between the matrix ${A}^+$ and the Sinkhorn structure.
First, the row-wise softmax of the logits $Z^+$ is  equivalent to the Gibbs kernel divided by its row sums. Let $r_i = \sum_m K_{im}$ be the $i$-th row sum:
\begin{equation}
    \text{softmax}_{\text{row}}(Z^+)_{ij} = \frac{\exp(Z^+_{ij})}{\sum_m \exp(Z^+_{im})} = \frac{K_{ij}}{r_i}
\end{equation}
Similarly, the column-wise softmax is the Gibbs kernel divided by its column sums. Let $c_j = \sum_n K_{nj}$ be the $j$-th column sum:
\begin{equation}
    \text{softmax}_{\text{col}}(Z^+)_{ij} = \frac{\exp(Z^+_{ij})}{\sum_n \exp(Z^+_{nj})} = \frac{K_{ij}}{c_j}
\end{equation}
The update of \citeauthor{deng2026drifting} corresponds to a geometric mean of these normalization operations,\footnote{In \eqref{eq:one_shot} there is additional scaling by $N^+$ and $N^-$, which we omit here - see discussion in main text.}
\begin{equation}
    {A}^+_{ij} = \frac{K_{ij}}{\sqrt{r_i \cdot c_j}}
\end{equation}
rather than iterating between them. 
When $N^+=N^-$, then repeated iteration of this procedure  converges to the entropy-regularised transport map  \citep[][Theorem 3.1]{knight2014symmetry}, as with the better known  Sinkhorn approach that alternates row and column normalization.

While the above correction does not guarantee the marginal constraints of a fully converged Sinkhorn approach (which would require infinite sequential iterations), it does capture the property of ``hub suppression:''
If a target $y_j$ is a hub (a region of high density surrounded by many source points), its column sum $c_j$ will be large. Consequently, the column scaling factor $ 1/\sqrt{c_j}$ becomes  small, which down-scales the affinity of  particles to that hub.

\section{Failure mode on a synthetic example}
\label{sec:failureSyntheticProofs}

We first describe the velocity field  of the KL-KDE in the regime where $D,\alpha,\beta$ are fixed and $\tau\downarrow0$.

\KLfailuremode*

\begin{proof}
The continuous KL-KDE drift evaluates the difference between the scores of the kernel-smoothed distributions: 
\[ V_{p,q}^\textup{KL}(x)=\frac{\tau}{2}\big(\nabla_x \log p_{\tau}(x)-\nabla_x \log q_{\tau}(x)\big). \]

We evaluate the score of the target distribution $p_{\tau}$ at $x=+D$. The smoothed density is
\[ p_{\tau}(D) = (1-\alpha)k(D,D) + \alpha k(D,-D) = (1-\alpha) + \alpha\epsilon. \]
The gradient of the unnormalized Gaussian kernel $k(x,z)=\exp(-\|x-z\|^2/\tau)$ is $\nabla_x k(x,z) = -\frac{2(x-z)}{\tau} k(x,z)$. Evaluating this at $x=+D$ gives $\nabla_x k(D,D)=0$ and $\nabla_x k(D,-D)=-\frac{4D}{\tau}\epsilon$. The  gradient of the smoothed target distribution is
\[ \nabla_x p_{\tau}(D) = (1-\alpha)\nabla_x k(D,D) + \alpha \nabla_x k(D,-D) = \alpha\left(-\frac{4D}{\tau}\epsilon\right). \]
The score function is the ratio of these two quantities:
\[ \nabla_x \log p_{\tau}(D) = \frac{\nabla_x p_{\tau}(D)}{p_{\tau}(D)} = \frac{-\frac{4D}{\tau}\alpha\epsilon}{(1-\alpha)+\alpha\epsilon}. \]

Assuming the mixing proportion $\alpha \neq 1$), 
\[ \frac{1}{(1-\alpha)+\alpha\epsilon} = \frac{1}{1-\alpha} \left( \frac{1}{1 + \frac{\alpha}{1-\alpha}\epsilon} \right). \]
Expanding $(1+z)^{-1} = 1 - z + \mathcal{O}(z^2)$ for small $z$, 
\begin{align*}
\nabla_x \log p_{\tau}(D) &= \left(-\frac{4D}{\tau}\alpha\epsilon\right) \frac{1}{1-\alpha} \left(1 - \frac{\alpha}{1-\alpha}\epsilon + \mathcal{O}(\epsilon^2)\right) \\
&= -\frac{4D}{\tau}\epsilon\left(\frac{\alpha}{1-\alpha}\right) + \mathcal{O}\left(\frac{\epsilon^2}{\tau}\right).
\end{align*}

To verify that the remainder  $\mathcal{O}(\frac{\epsilon^2}{\tau})$ is indeed $o(\epsilon)$, we must show that $\lim_{\tau \to 0} \frac{\epsilon^2/\tau}{\epsilon} = 0$. Using the fact that $\epsilon = \exp(-4D^2/\tau)$, we can express $\tau$ as $\tau = -\frac{4D^2}{\ln \epsilon}$. Then:
\begin{equation}\label{eq:whyEpsSqTauVanish}
\lim_{\epsilon \to 0} \frac{\epsilon^2/\tau}{\epsilon} = \lim_{\epsilon \to 0} \frac{\epsilon}{\tau} = \lim_{\epsilon \to 0} \frac{-\epsilon \ln \epsilon}{4D^2} = 0.
\end{equation}
Since the limit is  zero, the remainder  decays  faster than $\epsilon$. Therefore, $\mathcal{O}(\frac{\epsilon^2}{\tau}) \subset o(\epsilon)$, and thus,
\begin{align*}
\frac{\tau}{2} \nabla_x \log p_{\tau}(D) &= \frac{\tau}{2} \left[ -\frac{4D}{\tau}\epsilon\left(\frac{\alpha}{1-\alpha}\right) + \mathcal{O}\left(\frac{\epsilon^2}{\tau}\right) \right] \\
&= -2D\epsilon\left(\frac{\alpha}{1-\alpha}\right) +  o(\epsilon).
\end{align*}

Applying the same reasoning to the model distribution $q_{\tau}$, with mixing proportion $\beta$, yields:
\[ \frac{\tau}{2} \nabla_x \log q_{\tau}(D) = -2D\epsilon\left(\frac{\beta}{1-\beta}\right) + o(\epsilon). \]

Subtracting the two scaled scores gives the net KL-KDE drift:
\begin{align*}
V_{p,q}^\textup{KL}(D) &= \frac{\tau}{2} \nabla_x \log p_{\tau}(D) - \frac{\tau}{2} \nabla_x \log q_{\tau}(D) \\
&= \left[ -2D\epsilon\left(\frac{\alpha}{1-\alpha}\right) + o(\epsilon) \right] - \left[ -2D\epsilon\left(\frac{\beta}{1-\beta}\right) + o(\epsilon) \right] \\
&= -2D\epsilon\left(\frac{\alpha}{1-\alpha} - \frac{\beta}{1-\beta}\right) + o(\epsilon).
\end{align*}

\end{proof}

The drift points in the correct direction (negative, since we assume $\beta < \alpha$), but its magnitude is scaled by $\epsilon$. Because $\tau$ is narrow relative to the mode separation $2D$, $\epsilon$ may drop below machine precision ($\epsilon \approx 0$).
We now show that the Sinkhorn proxy method has the  same failure mode.

\propFailureMode*

\begin{proof}
We begin with the population-limit velocity field established in \cref{thm:scoreInterpretationDrifting},
\[ V_{p,q}^\mathrm{SP}(x)=\frac{\tau}{2}Z(x)\big(\nabla_x \log(k*\tilde{p})(x)-\nabla_x \log(k*\tilde{q})(x)\big) , \]
where  $d\tilde{p}(y) = q^{-1/2}_\tau(y) dp(y)$ and $d\tilde{q}(x) = q^{-1/2}_\tau(x) dq(x)$. We evaluate the algorithmic drift $V_{p,q}^\mathrm{SP}(x)$ at the right-hand overcrowded mode $x=+D$.

First, we evaluate the kernel-smoothed model density $q_{\tau}$  at both modes, expressing the result as $1+\mathcal{O}(\epsilon)$ perturbations:
\begin{align*}
q_{\tau}(D) &= \beta k(D,-D)+(1-\beta)k(D,D)=(1-\beta)+\beta\epsilon=(1-\beta)(1+\mathcal{O}(\epsilon)) \\
q_{\tau}(-D) &= \beta k(-D,-D)+(1-\beta)k(-D,D)=\beta+(1-\beta)\epsilon=\beta(1+\mathcal{O}(\epsilon))
\end{align*}

Using the Taylor expansion $(1+\mathcal{O}(\epsilon))^{-1/2} = 1+\mathcal{O}(\epsilon)$, we substitute these into the modified distributions $\tilde{p}$ and $\tilde{q}$:
\begin{align*}
d\tilde{p}(y) &= \frac{\alpha}{\sqrt{q_{\tau}(-D)}}\delta_{-D}(y)+\frac{1-\alpha}{\sqrt{q_{\tau}(D)}}\delta_{+D}(y) = \left(\frac{\alpha}{\sqrt{\beta}}+\mathcal{O}(\epsilon)\right)\delta_{-D}(y)+\left(\frac{1-\alpha}{\sqrt{1-\beta}}+\mathcal{O}(\epsilon)\right)\delta_{+D}(y) \\
d\tilde{q}(x) &= \frac{\beta}{\sqrt{q_{\tau}(-D)}}\delta_{-D}(x)+\frac{1-\beta}{\sqrt{q_{\tau}(D)}}\delta_{+D}(x) = \left(\sqrt{\beta}+\mathcal{O}(\epsilon)\right)\delta_{-D}(x)+\left(\sqrt{1-\beta}+\mathcal{O}(\epsilon)\right)\delta_{+D}(x)
\end{align*}

Next, we evaluate the kernel-smoothed modified measures at $x=+D$. Since $k(D,-D)=\epsilon$ and $k(D,D)=1$:
\begin{equation}\label{eq:blurredTildePQ}
\begin{aligned}
(k*\tilde{p})(D) &= \left(\frac{\alpha}{\sqrt{\beta}}+\mathcal{O}(\epsilon)\right)\epsilon+\left(\frac{1-\alpha}{\sqrt{1-\beta}}+\mathcal{O}(\epsilon)\right)(1) = \frac{1-\alpha}{\sqrt{1-\beta}}+\mathcal{O}(\epsilon) \\
(k*\tilde{q})(D) &= \left(\sqrt{\beta}+\mathcal{O}(\epsilon)\right)\epsilon+\left(\sqrt{1-\beta}+\mathcal{O}(\epsilon)\right)(1) = \sqrt{1-\beta}+\mathcal{O}(\epsilon)
\end{aligned}
\end{equation}
To compute the score functions, we evaluate the  gradients at $x=+D$. For the unnormalized Gaussian kernel $k(x,z)=\exp(-\|x-z\|^2/\tau)$, the  gradients are $\nabla_x k(x,-D)|_{x=D}=-\frac{4D}{\tau}\epsilon$ and $\nabla_x k(x,D)|_{x=D}=0$. This yields:
\begin{align*}
\nabla_x(k*\tilde{p})(D) &= \left(\frac{\alpha}{\sqrt{\beta}}+\mathcal{O}(\epsilon)\right)\left(-\frac{4D}{\tau}\epsilon\right)+\left(\frac{1-\alpha}{\sqrt{1-\beta}}+\mathcal{O}(\epsilon)\right)(0) = -\frac{4D}{\tau}\epsilon\frac{\alpha}{\sqrt{\beta}}+\mathcal{O}\left(\frac{\epsilon^2}{\tau}\right) \\
\nabla_x(k*\tilde{q})(D) &= \left(\sqrt{\beta}+\mathcal{O}(\epsilon)\right)\left(-\frac{4D}{\tau}\epsilon\right)+\left(\sqrt{1-\beta}+\mathcal{O}(\epsilon)\right)(0) = -\frac{4D}{\tau}\epsilon\sqrt{\beta}+\mathcal{O}\left(\frac{\epsilon^2}{\tau}\right)
\end{align*}

The score functions at $x=+D$ are then the ratio of the gradients to the densities. Using the quotient Taylor expansion $\frac{A\epsilon + \mathcal{O}(\epsilon^2/\tau)}{B + \mathcal{O}(\epsilon)} = \frac{A}{B}\epsilon + \mathcal{O}\left(\frac{\epsilon^2}{\tau}\right)$ and  that $\mathcal{O}(\epsilon^2/\tau) = o(\epsilon)$ from \eqref{eq:whyEpsSqTauVanish},
\begin{align*}
\nabla_x \log(k*\tilde{p})(D) &= \frac{-\frac{4D}{\tau}\epsilon\frac{\alpha}{\sqrt{\beta}}+\mathcal{O}\left(\frac{\epsilon^2}{\tau}\right)}{\frac{1-\alpha}{\sqrt{1-\beta}}+\mathcal{O}(\epsilon)} = -\frac{4D}{\tau}\epsilon\left(\frac{\alpha}{1-\alpha}\right)\sqrt{\frac{1-\beta}{\beta}}+o(\epsilon) \\
\nabla_x \log(k*\tilde{q})(D) &= \frac{-\frac{4D}{\tau}\epsilon\sqrt{\beta}+\mathcal{O}\left(\frac{\epsilon^2}{\tau}\right)}{\sqrt{1-\beta}+\mathcal{O}(\epsilon)} = -\frac{4D}{\tau}\epsilon\sqrt{\frac{\beta}{1-\beta}}+o(\epsilon)
\end{align*}
Thus
\begin{equation}
\frac{\tau}{2} \left( \nabla_x \log(k*\tilde{p})(D) - \nabla_x \log(k*\tilde{q})(D) \right)
=
-2D \epsilon 
\sqrt{\frac{1-\beta}{\beta}}
\left[
\frac{\alpha}{1-\alpha} 
-
\frac{\beta}{1-\beta}
\right] + \frac{\tau}{2} o(\epsilon) 
\end{equation}

Finally, we consider the scaling term 
\begin{equation}
Z(D) = s^+(D)s^-(D) 
=
\frac{1}{\sqrt{p_\tau(D)}}(k*\tilde{p})(D) \cdot
\frac{1}{\sqrt{q_\tau(D)}}(k*\tilde{q})(D)
\end{equation}
We have
\begin{equation}
\frac{1}{\sqrt{p_\tau(D)}} = \frac{1}{\sqrt{1-\alpha}} + \mathcal{O}(\epsilon), 
\qquad
\frac{1}{\sqrt{q_\tau(D)}} = \frac{1}{\sqrt{1-\beta}} + \mathcal{O}(\epsilon),
\end{equation}
and thus, from previous results \eqref{eq:blurredTildePQ} on $(k*\tilde{p})(D)$ and $(k*\tilde{q})(D)$, that
\begin{equation}
s^+(D) = \frac{\sqrt{1-\alpha}}{\sqrt{1-\beta}} + \mathcal{O}(\epsilon),
\qquad
s^-(D) = 1 + \mathcal{O}(\epsilon). 
\end{equation}
It follows that 
\begin{equation}
Z(D) = \frac{\sqrt{1-\alpha}}{\sqrt{1-\beta}} + \mathcal{O}(\epsilon).
\end{equation}

Substituting these into the  drift equation gives:
\begin{equation}
V_{p,q}^\mathrm{SP}(D) 
=
-2D\epsilon\left[\sqrt{\frac{1-\alpha}{\beta}}\right]\left(\frac{\alpha}{1-\alpha}-\frac{\beta}{1-\beta}\right)+o(\epsilon).
\end{equation}
\end{proof}

This confirms that in the present setting, despite its superficial resemblance to a Sinkhorn flow, the Sinkhorn Proxy algorithm does not share the mass transport dynamics of an Earth Mover's Distance, but behaves more like a flow on KL. We suspect that \citeauthor{deng2026drifting}'s Algorithm 2 also has this problem.

\trueWassersteinGF*
\begin{proof}
 We compute the barycentric velocity at $x = +D$ for the Wasserstein gradient flow of $\frac12 W_2^2(q,p)$. Because the measures are atomic and the optimal plan can split the mass at $+D$, this should be understood as the barycentric projection of the optimal plan, or as the limit obtained by replacing the atoms by small absolutely continuous clusters. As defined in the main text, the target data distribution $p$ and the model distribution $q$ are mixtures of two isolated point masses:
\begin{align}
    p(y) &= \alpha \updelta_{-D}(y) + (1-\alpha) \updelta_{+D}(y) , \\
    q(x) &= \beta \updelta_{-D}(x) + (1-\beta) \updelta_{+D}(x) . 
\end{align}
We assume $\beta < \alpha$. This means the model $q$ is ``starved'' of mass at the left mode ($-D$) and ``overcrowded'' at the right mode ($+D$).

To find the Wasserstein-2 velocity, we must first find the  optimal transport plan $\pi(x, y)$ that minimizes the quadratic cost $C(x, y) = \frac{1}{2}\|x - y\|^2$ to move mass from $q$ to $p$.
The joint distribution $\pi(x, y)$ must satisfy the following marginal constraints:
\begin{equation}
    \sum_y \pi(x, y) = q(x), \qquad \sum_x \pi(x, y) = p(y) .
\end{equation}

To minimize the cost, mass should remain in place wherever possible (cost = 0).
\begin{itemize}
    \item All mass $\beta$ at $x = -D$ in $q$ will stay at $y = -D$ in $p$. Thus, $\pi(-D, -D) = \beta$.
    \item Because the target $p$ requires a total mass of $\alpha$ at $-D$, the remaining deficit of $(\alpha - \beta)$ must be transported from $x = +D$. Thus, $\pi(+D, -D) = \alpha - \beta$.
    \item The remaining mass at $x = +D$ stays at $y = +D$. Thus, $\pi(+D, +D) = (1-\beta) - (\alpha - \beta) = 1 - \alpha$.
    \item Finally, to obtain that the probability sum at one, we have $\pi(-D,+D) = 0$.
\end{itemize}

For probability mass at $x = +D$, the Wasserstein gradient flow velocity is
\begin{equation} \label{eq:expected_gradient}
    V^{W_2}_{p,q}(x) = -\int \nabla_x C(x,y) \, d\pi(y|x),  
\end{equation}
with the expectation over the optimal \textit{conditional} transport plan $d\pi(y|x)$ \citep{ambrosio2008gradient}.

We compute the conditional probability $\pi(y \mid x) = \frac{\pi(x, y)}{q(x)}$ for the overcrowded mode $x = +D$:

The total available mass at $x = +D$ is $q(+D) = 1 - \beta$.
\begin{itemize}
    \item Fraction of mass moving to $-D$: $\pi(-D \mid +D) = \frac{\alpha - \beta}{1 - \beta}$
    \item Fraction of mass staying at $+D$: $\pi(+D \mid +D) = \frac{1 - \alpha}{1 - \beta}$
\end{itemize}
The velocity vector $V^{W_2}_{p,q}$ for the mass at $x$ is the difference between its expected target destination and its current position:
\begin{equation}
    V^{W_2}_{p,q}(x) = \mathbb{E}_{Y \sim \pi(\cdot \mid x)}[Y] - x
\end{equation}
Evaluating this at $x = +D$:
\begin{equation}
    V^{W_2}_{p,q}(+D) = \left[ \pi(-D \mid +D)(-D) + \pi(+D \mid +D)(+D) \right] - D.
\end{equation}
Substituting the conditional probabilities,
\begin{align}
   V^{W_2}_{p,q}(+D) &= \left[ \left( \frac{\alpha - \beta}{1 - \beta} \right)(-D) + \left( \frac{1 - \alpha}{1 - \beta} \right)(+D) \right] - D \\
   &=  D \left[ \frac{-(\alpha - \beta) + (1 - \alpha) - (1 - \beta)}{1 - \beta} \right] \\
    &= -2D \left( \frac{\alpha - \beta}{1 - \beta} \right)
\end{align}
which concludes the proof.

\end{proof}

\section{Proof of Proposition~\ref{prop:loss_gradient}}
\label{sec:loss_gradient_proof}

\propLossGradient*
\begin{proof}
Write $x = f_\theta(\epsilon)$ and let $t = \sg(x + \eta V_{p,q_\theta}(x))$ denote the stop-gradient target, which is treated as a constant with respect to $\theta$. Then $\mathcal{L}(\theta) = \E_{\epsilon \sim \mu}[\|f_\theta(\epsilon) - t\|^2]$, and applying the chain rule yields:
\begin{align}
    \nabla_\theta \mathcal{L}(\theta) 
    &= \E_{\epsilon \sim \mu}\Big[2\big(f_\theta(\epsilon) - t\big)^\top \nabla_\theta f_\theta(\epsilon)\Big] \nonumber\\
    &= \E_{\epsilon \sim \mu}\Big[2\big(f_\theta(\epsilon) - f_\theta(\epsilon) - \eta V_{p,q_\theta}(f_\theta(\epsilon))\big)^\top \nabla_\theta f_\theta(\epsilon)\Big] \nonumber\\
    &= -2\eta\,\E_{\epsilon \sim \mu}\Big[V_{p,q_\theta}\big(f_\theta(\epsilon)\big)^\top \nabla_\theta f_\theta(\epsilon)\Big].
\end{align}
A perturbation $\delta\theta$ changes the functional through the marginal distribution $q_\theta$. By the definition of the first variation:
\begin{equation}
    \nabla_\theta \cF_p(q_\theta) = \int \frac{\delta \cF_p}{\delta q_\theta}(x)\, \nabla_\theta q_\theta(x)\, \dd x.
\end{equation}
Since $q_\theta = (f_\theta)_\# \mu$, the parameter perturbation $\delta\theta$ induces a displacement field on the particles in data space. Because multiple noise instances $\epsilon$ could map to the same spatial point $x$ (if $f_\theta$ is non-invertible), we define the effective Eulerian velocity field at spatial location $x$ via the conditional expectation:
\begin{equation}
    v_\theta(x) = \E_{\epsilon \sim \mu}\big[\nabla_\theta f_\theta(\epsilon) \mid f_\theta(\epsilon) = x\big].
\end{equation}
The evolution of the density $q_\theta$ with respect to $\theta$ is mathematically governed by the continuity equation $\nabla_\theta q_\theta(x) = -\nabla_x \cdot (q_\theta(x) v_\theta(x))$. Substituting this into the first variation and applying the divergence theorem (assuming boundary terms vanish) gives:
\begin{align}
    \nabla_\theta \cF_p(q_\theta) &= -\int \frac{\delta \cF_p}{\delta q_\theta}(x)\, \nabla_x \cdot \big(q_\theta(x)\, v_\theta(x)\big)\, \dd x \nonumber\\
    &= \int \nabla_x \frac{\delta \cF_p}{\delta q_\theta}(x)^\top v_\theta(x) \; q_\theta(x)\, \dd x.
\end{align}
Using the law of total expectation (the tower property), we can rewrite this integral over the spatial domain $\mathbb{R}^d$ as an expectation directly over the latent noise distribution $\mu$:
\begin{align}
    \nabla_\theta \cF_p(q_\theta) 
    &= \E_{x \sim q_\theta}\Big[ \nabla_x \frac{\delta \cF_p}{\delta q_\theta}(x)^\top \E_{\epsilon \sim \mu}\big[\nabla_\theta f_\theta(\epsilon) \mid f_\theta(\epsilon) = x\big] \Big] \nonumber\\
    &= \E_{\epsilon\sim\mu}\Big[\nabla_x \frac{\delta \cF_p}{\delta q_\theta}\big(f_\theta(\epsilon)\big)^\top \nabla_\theta f_\theta(\epsilon)\Big]. \label{eq:functional_chain_rule}
\end{align}

Finally, substituting the definition of the Wasserstein gradient $V_{p,q_\theta}(x) = -\nabla_x \frac{\delta \cF_p}{\delta q_\theta}(x)$ into \eqref{eq:functional_chain_rule} directly recovers the relation $\nabla_\theta \mathcal{L}(\theta) = 2\eta\,\nabla_\theta \cF_p(q_\theta)$.
\end{proof}

\section{Additional results}
\label{sec:additional_results}
In Figure \ref{fig:gaussians_results} to Figure \ref{fig:swiss_roll_results}, we present experimental results for various toy examples.

\begin{figure}[ht]
    \centering
    \includegraphics[width=0.9\linewidth]{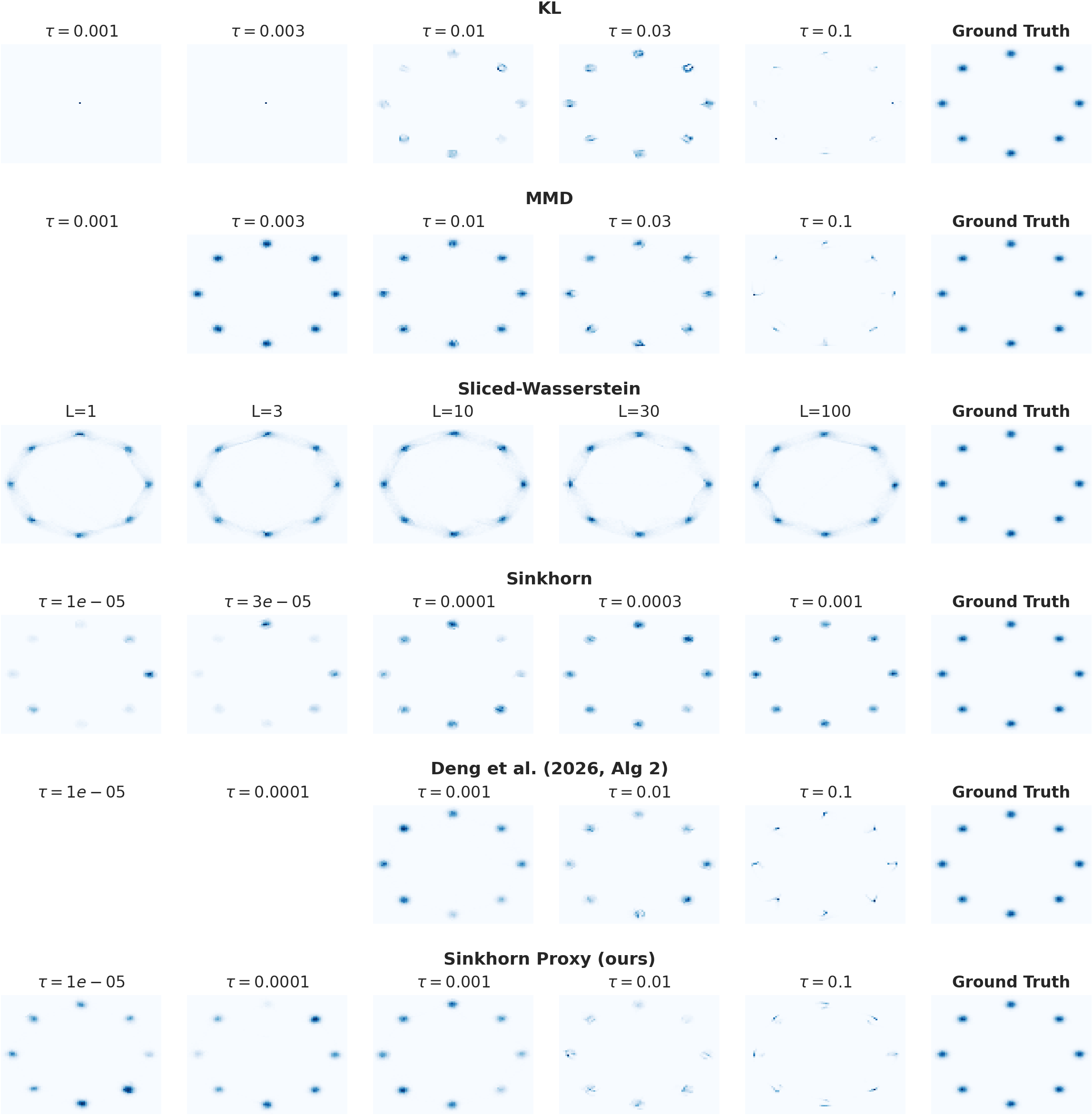}
    \caption{Results for the 8 Gaussian dataset. Empty panel means the samples have diverged.}
    \label{fig:gaussians_results}
\end{figure}

\begin{figure}
    \centering
    \includegraphics[width=0.95\linewidth]{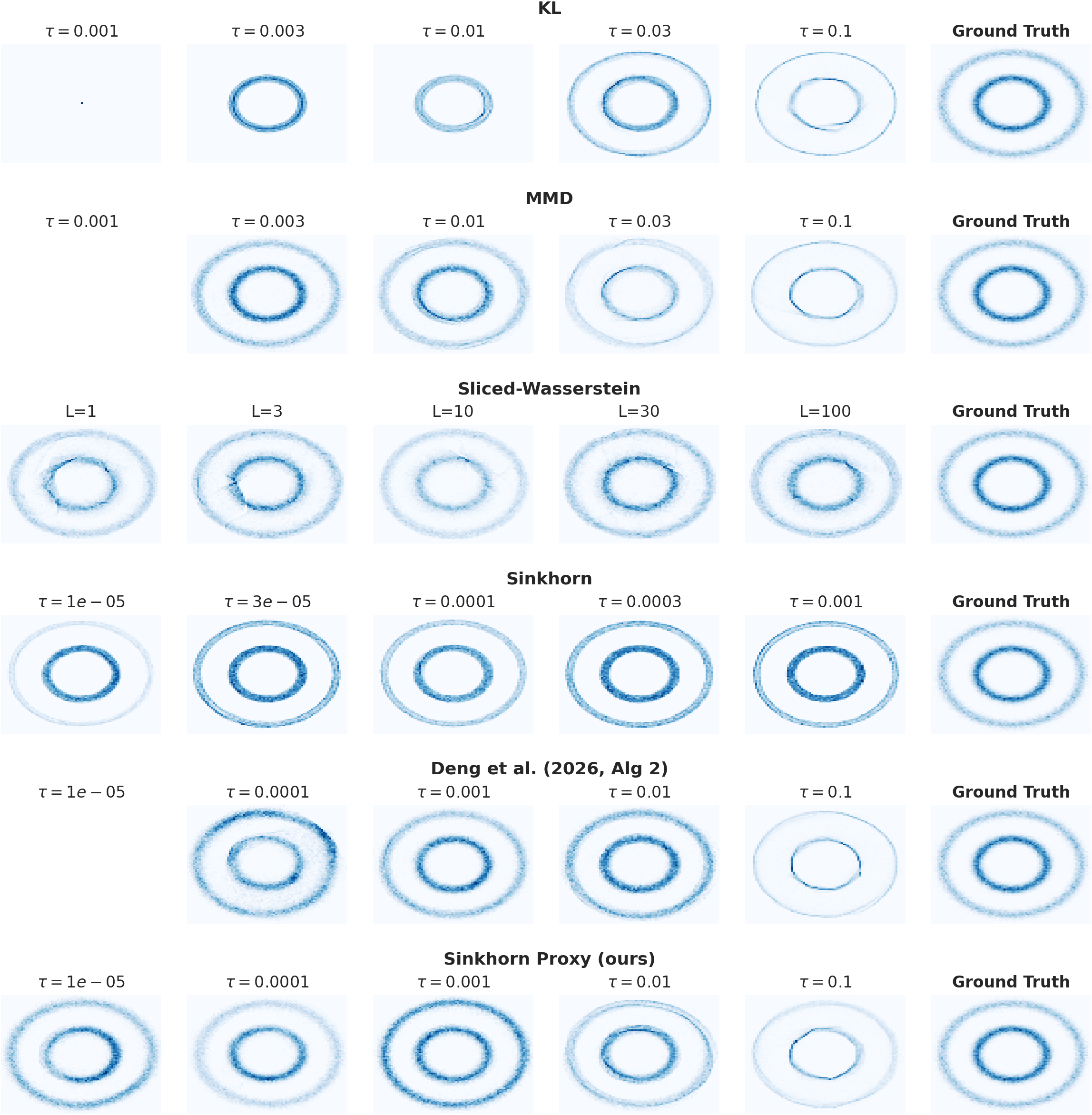}
    \caption{Results for the Circles dataset. Empty panel means the samples have diverged.}
    \label{fig:circles_results}
\end{figure}

\begin{figure}
    \centering
    \includegraphics[width=0.95\linewidth]{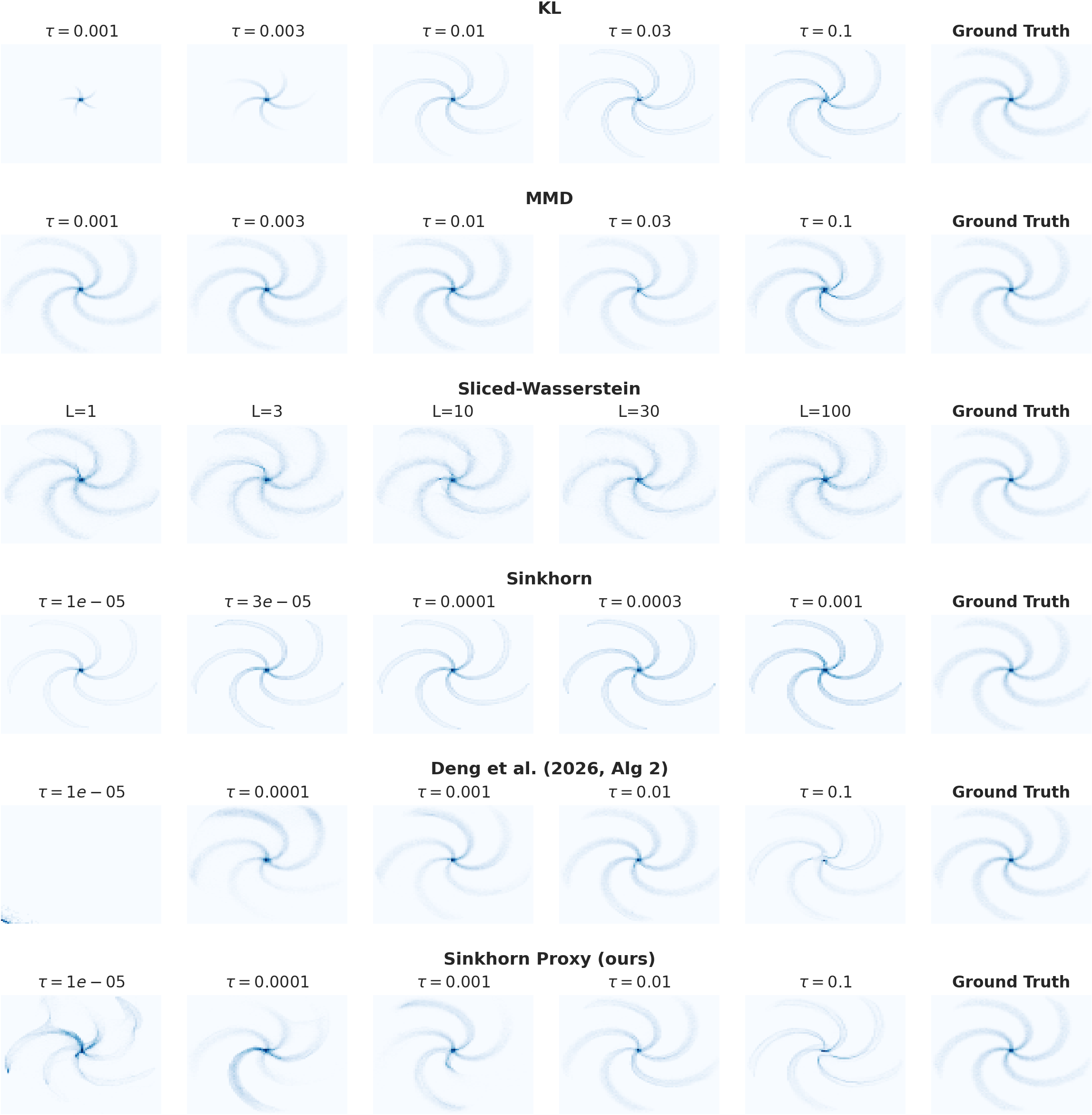}
    \caption{Results for the Pinwheel dataset. Empty panel means the samples have diverged.}
    \label{fig:pinwheel_results}
\end{figure}

\begin{figure}
    \centering
    \includegraphics[width=0.95\linewidth]{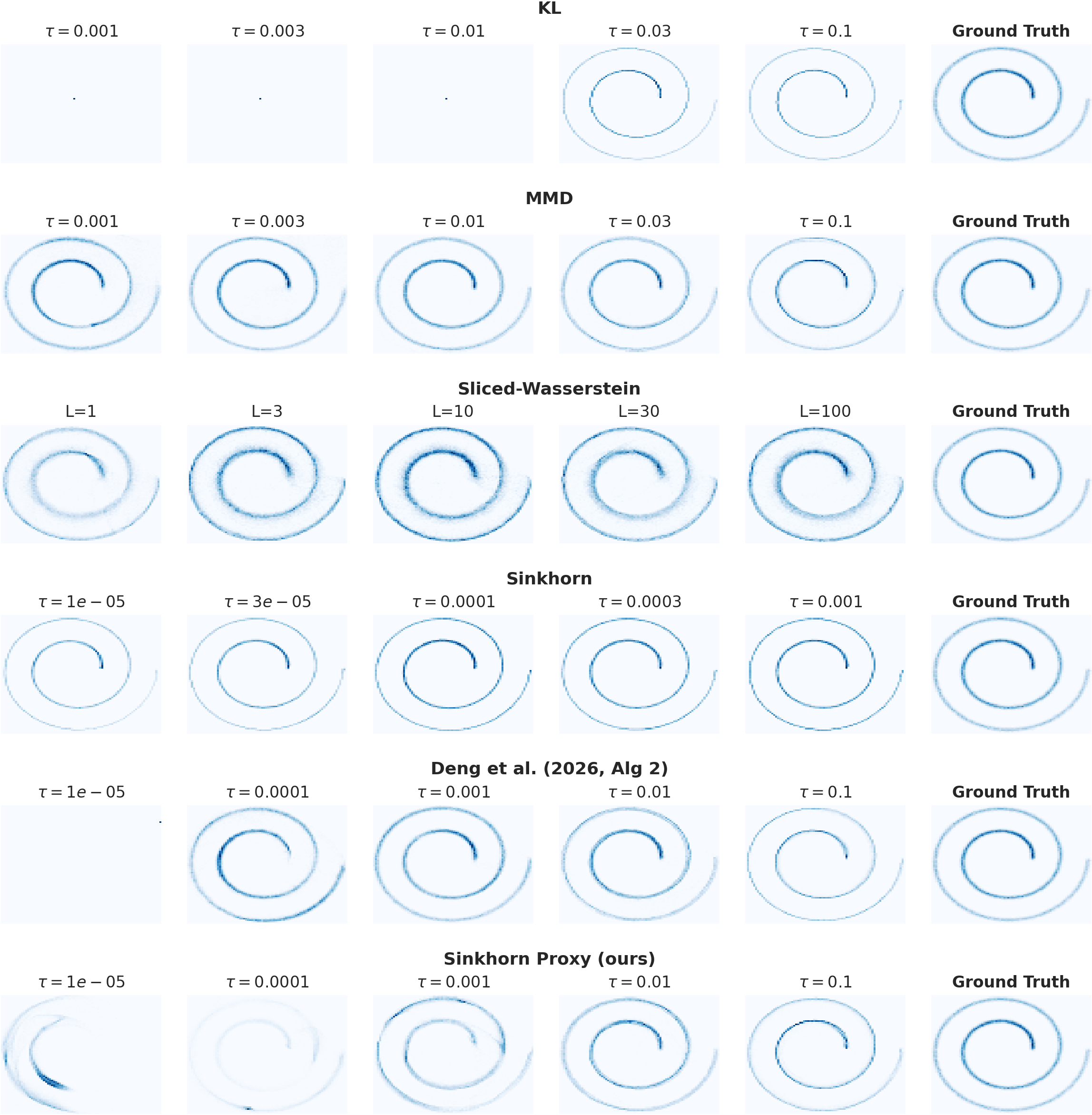}
    \caption{Results for the Swiss roll dataset. Empty panel means the samples have diverged.}
    \label{fig:swiss_roll_results}
\end{figure}

\end{document}